\documentclass[a4paper]{article}

\usepackage[utf8]{inputenc} 
\usepackage[T1]{fontenc}    
\usepackage{url}            
\usepackage{booktabs}       
\usepackage{amsfonts}       
\usepackage{nicefrac}       
\usepackage{microtype}      
\usepackage[usenames,dvipsnames]{xcolor}
\usepackage{natbib}

\usepackage{tikz}
\usetikzlibrary{decorations.markings,calc,backgrounds,intersections}
\usepackage{adjustbox}
\usepackage{graphicx}
\usepackage{geometry}
\usepackage{indentfirst} 
\usepackage{amsmath,amsthm,amsfonts,amssymb} 
\usepackage[normalem]{ulem} 
\usepackage{bm} 
\usepackage{color} 
\usepackage{makeidx} 
\usepackage[linewidth=1pt,nobreak=true,framemethod=TikZ]{mdframed} 
\usepackage[titletoc]{appendix} 
\usepackage{enumitem} 
\usepackage{tikz-cd}  

\usepackage{authblk}

\usepackage[colorlinks]{hyperref} 
\allowdisplaybreaks[4]  
\newtheorem{theorem}{Theorem}[section]  
\newtheorem{lemma}[theorem]{Lemma}  
\newtheorem{proposition}[theorem]{Proposition}  
\newtheorem{definition}[theorem]{Definition}  
\newtheorem{corollary}[theorem]{Corollary}  

\title{Constructing 3D Rotational Invariance and Equivariance \\
  with Symmetric Tensor Networks}

\begin{document}
\author[1]{Meng Zhang}
\author[1]{Chao Wang}
\author[5]{Hao Zhang}
\author[1]{Shaojun Dong}
\author[2,1,3,4]{Lixin He \thanks{Corresponding to: helx@ustc.edu.cn}}

\affil[1]{Institute of Artificial Intelligence, Hefei Comprehensive National Science Center, Hefei, 230088, People's Republic of China}
\affil[2]{CAS Key Laboratory of Quantum Information, University of Science and Technology of China, Hefei 230026, People's Republic of China}
\affil[3]{Synergetic Innovation Center of Quantum Information and Quantum Physics, University of Science and Technology of China, Hefei 230026, China}
\affil[4]{Hefei National Laboratory, University of Science and Technology of China, Hefei 230088, China}
\affil[5]{School of Information Science and Technology, University of Science and Technology of China, Hefei, 230026, China}

\date{} 
\maketitle

\begin{abstract}
 Symmetry-aware architectures are central to geometric deep learning. We present a systematic approach for constructing continuous rotationally invariant and equivariant functions using symmetric tensor networks. The proposed framework supports inputs and outputs given as a tuple of Cartesian tensors of different rank as well as spherical tensors of different type. We introduce tensor network generators for invariant maps and obtain equivariant maps via differentiation. Specifically, we derive general continuous equivariant maps from vector inputs to Cartesian or spherical tensor output. Finally, we clarify how common equivariant primitives in geometric graph neural networks arise within our construction.
\end{abstract}

\section{Introduction}
Many scientific problems involve systems with 3D geometric structures, such as molecules and materials. For these systems, we can always describe them using 3D coordinates. However, physical quantities in the natural world do not depend on any specific coordinate system. They exhibit invariance or equivariance under changes of coordinates, such as rotations and translations. Therefore, when performing machine learning tasks involving physical quantities, it is advantageous in data efficiency and generalization to incorporate invariance or equivariance into the neural network hypotheses \citep{geiger2022e3nneuclideanneuralnetworks}.

Designing invariant or equivariant functions in neural networks architecture is a crucial step in developing the symmetry-aware machine learning models. For instance, in equivariant geometric  graph neural networks (GNN) with vector features, the equivariant operations on vector features  are typically vector summation $v_1 + v_2$ \citep{pmlr-v139-satorras21a,pmlr-v139-schutt21a,9711441} and vector product $v_1 \times v_2$ \citep{le2022equivariantgraphattentionnetworks}. For Tensor Field Network (TFN) with higher-type spherical tensors  features, the most commonly used equivariant operations are
the tensor product (TP) operations \citep{thomas2018tensorfieldnetworksrotation,NEURIPS2018_488e4104}. In addition, recent works also use the  higher-rank Cartesian tensors as the equivariant feature in the message passing, in which the typical equivariant operations used are tensor contraction and summation of tensors \citep{Wang2024}.

Beyond these practically motivated symmetry functions, there is a growing interest in a full characterisation of the group invariant and group equivariant neural networks that can be built for given input and output types. Prior theory has made important progress in special cases. For Cartesian tensors, work \cite{NEURIPS2021_f1b07759} characterize the $O(n)$ and $SO(n)$ invariant and equivariant functions with vector inputs and outputs; works \cite{gregory2024learningequivarianttensorfunctions} and \cite{10756158} show how to construct $O(n)$ symmetric polynomials of Cartesian tensors. Relatedly, work \cite{pmlr-v202-pearce-crump23a} uses the Brauer algebra \citep{brauer1937algebras} to derive $O(n)$ and $SO(n)$ equivariant linear maps between Cartesian tensors. On the spherical  tensor side, 3D steerable CNNs \citep{NEURIPS2018_488e4104} systematically handle spherical tensor features, but their characterisation is mainly focused on linear maps. These results offer valuable theoretical guidance, while leaving open the broader setting of more general rotational symmetric functions.

In this work, we present a full characterisation of continuous rotationally invariant and equivariant functions, which supports inputs and outputs given as a tuple of Cartesian tensors of different rank as well as spherical tensors of different type. The main tool we use is the symmetric tensor network \citep{PhysRevA.82.050301,PhysRevB.83.115125,PhysRevB.86.195114}, which is widely used in quantum many-body systems. Combining the classical invariance theory \citep{weyl1946classical} and Stone-Weierstrass theorem, we develop a framework for building generators of invariant polynomials, which we call \textbf{\textit{ tensor network generators}}, and demonstrate how the general continuous $SO(3)$ invariance and equivariance functions build from several special symmetric tensor primitives.

The main contributions of this work are:
\begin{enumerate}
\item  We introduce tensor network generators for invariant polynomials and show how to obtain equivariant polynomials from them by differentiation, capable of handling inputs and outputs composed of a tuple of Cartesian and spherical tensors.

\item We provide a complete characterisation of continuous $SO(3)$ invariant and equivariant functions, covering both linear and nonlinear maps, and show that they can be built from several special symmetric tensors combined with general functions that do not incorporate $SO(3)$ symmetry.

\item We derive the concrete and concise form of general continuous $SO(3)$ and $O(3)$ equivariant maps with vector inputs and Cartesian or spherical tensor output, which is the most prevalent scenario in scientific modeling.

\item We demonstrate that common equivariant primitives in geometric GNNs naturally emerge as special cases of our construction. Furthermore, our proposed graphical representation simplifies both the theoretical proofs and the practical design of symmetry aware architectures.

\end{enumerate}

\section{Preliminaries}
\subsection{Tensor network}
Tensor networks have been proven to be a powerful graphical language and computational tool across multiple disciplines. The roots of this diagrammatic notation can be traced back to the work of Roger Penrose in the 1970s \cite{penrose1971applications}. We further discuss related literature to situate our work within the broader field in Section \ref{re_wor}.

We first introduce the formalism of tensors, which are the building blocks of tensor networks. A tensor $T$ is a multi-dimensional array. We can denote its elements as $T_{i_{1},i_{2},\dots,i_{n}} \in \mathbb{R}^{I_{1} \times I_{2} \times \cdots \times I_{n}}$, where $I_k$ is the dimension of index $i_k$. The tensor also has a graphical representation. As shown in Fig.\ref{tns}(a), a tensor can be represented by a node with legs, where each leg corresponds
to an index of the tensor. A vector can be represented by a one-leg node, and a matrix can be represented by a two-leg node. 

A tensor network is a collection of tensors defined above. The legs connected between nodes are the indices 
needed to be summed over, this is called contraction.  Therefore, the  tensor network can be contracted to a single tensor, the index of which corresponds to the open leg of the tensor network. As shown in Fig.\ref{tns}(b), the tensor network describes the contraction of tensor $A$ and $B$, i.e., matrix multiplication. Furthermore, we can also give the  graphical representation of the derivative of the tensor network. As shown in Fig.\ref{tns}(c), the derivative of a tensor network with respect to a specific tensor $T$ (which appears only once in the network) is a tensor network where tensor $T$ is removed. 
We call a tensor network connected if its corresponding graph is connected. Otherwise, it is disconnected, and its contraction factorizes as the product of its connected components.

Actually, a tensor network represents a certain decomposition of a high-rank tensor. As shown in Fig.\ref{tns}(d), a N-leg tensor is decomposed into 2-leg and 3-leg tensors, where this decomposition is called tensor train decomposition \citep{PhysRevLett.91.147902,PhysRevB.73.094423,doi:10.1137/090752286}.

\begin{figure}[htb!]
    \centering
    \includegraphics[width=0.7\linewidth]{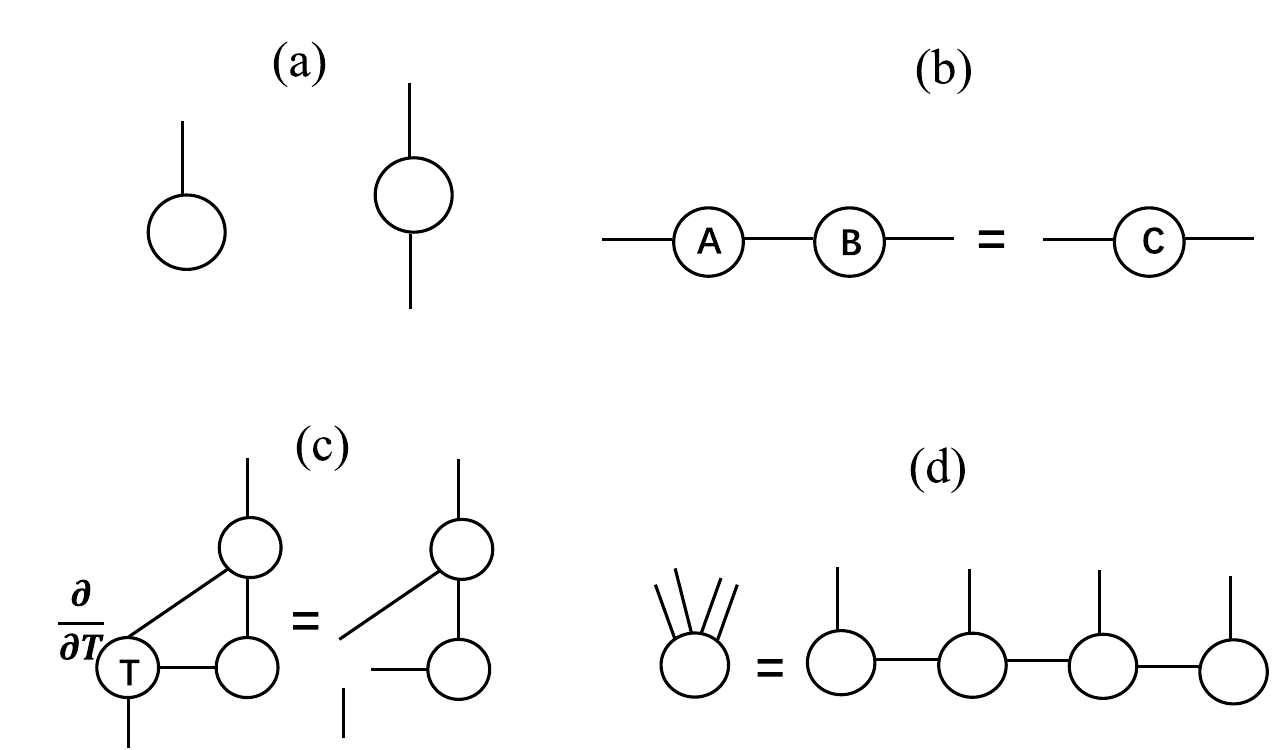}
    \caption{(a) Graphical representation of a one-leg tensor (vector) and a two-leg tensor (matrix). (b) The contraction of tensor $A$ and $B$, this is the matrix multiplication $C_{ik} = \sum_j A_{ij}B_{jk}$. (c) The derivative of a tensor network with respect to a specific tensor $T$, the result of which is a tensor network where tensor $T$ is removed. (d) Tensor train decomposition, namely, a N-leg tensor is decomposed to 2-leg and 3-leg tensors.
    }
    \label{tns}
\end{figure}

\subsection{Group invariance and equivariance}
We can give the definition of group invariance and equivariance functions as follows:
\begin{definition}
Let $G$ be a group which acts on linear spaces $V_1,\dots,V_n$ over a field $F$ by certain linear representation. 
An invariant function $f: \bigoplus_i V_i \rightarrow F$ is a multi-variable function such that for each $g\in G$,
\begin{equation}
  f(g \cdot \bm x_1, \cdots, g \cdot \bm x_n) = f(\bm x_1, \cdots, \bm x_n)
\end{equation}
\end{definition}

\begin{definition}
Let $G$ be a group which acts on linear spaces $V_1,\dots,V_n,U_1,\dots,U_m$ over $F$ by certain linear representation. An equivariant function $f: \oplus_j V_j \rightarrow \oplus_i U_i$ is a multi-variable function such that for each $g\in G$,
\begin{equation}
  f^i(g \cdot \bm x_1, \cdots , g \cdot \bm x_n) = g \cdot  f^i(\bm x_1, \cdots,\bm x_n)
\end{equation}
\end{definition}
where the $\cdot$ denotes the group action on a  linear space. For example, the inputs of the equivariance functions can be 3D coordinates of each atom in a molecule, and the outputs can be the force of each atom. In the remaining part of the  paper, we mainly focus on the $SO(3)$ group.

\subsection{Symmetric tensor network}

A symmetric tensor \citep{PhysRevA.82.050301,PhysRevB.83.115125,PhysRevB.86.195114} is a tensor that is invariant under a group action in the space of each of its indices. 

\begin{definition}
	Let $T_{i_1,\dots,i_n}\in \mathbb{R}^{I_{1} \times  \cdots \times I_{n}}$ be a tensor, and $\rho_k: G\mapsto  GL(I_k,\mathbb R)$ be the group representation on space of $k$-th index. $T_{i_1,\dots,i_n}$ is called a (group) {\bf  symmetric tensor} iff
	\begin{equation}
		\forall g\in G: \prod_k^{n} \left[\rho_k(g) \right]_{i_k,j_k}T_{i_1,\dots,i_n}=T_{j_1,\dots,j_n}
		\label{eqn:sym_cond}
	\end{equation}
\end{definition}
This can be illustrated by Fig. \ref{fig:symmetry-tensor}(a). A symmetric tensor network is a collection of symmetric tensors defined above. Fig. \ref{fig:symmetry-tensor}(b) illustrates that a network composed of symmetric tensors is itself symmetric. Thus, the contraction operation preserves tensor symmetry. This fundamental property allows us to employ a more restrictive formation for  tensors with predefined symmetries, namely, constructing the tensor network exclusively from symmetric tensors.

\begin{figure}[htb!]
    \centering
    \includegraphics[width=0.7\linewidth]{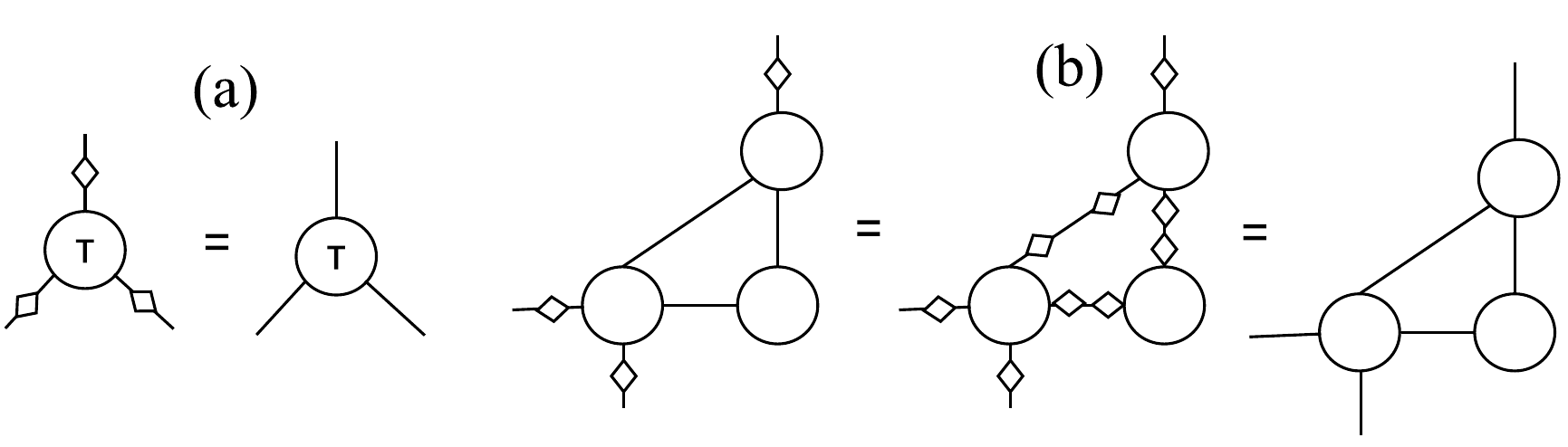}
    \caption{Symmetric tensors and symmetric tensor networks.
(a) The graphical illustration of the Equation (\ref{eqn:sym_cond}).
(b) Tensor networks that consist of symmetric tensors are also symmetric tensors as a whole. We first insert identity $\rho_i(g) \rho_i(g)^{T} = I$ on the contracted leg. Since every tensor in the network is symmetric, the tensor networks as a whole are also symmetric.
    }
    \label{fig:symmetry-tensor}
\end{figure}

\subsection{Cartesian and spherical tensors}

A \emph{Cartesian tensor} of rank $r$ is an element of the tensor product space
$T \in (\mathbb{R}^3)^{\otimes r}$.
Given a rotation $R \in SO(3)$, the group acts on $T$ by rotating each index,
$\left [R \cdot T \right]_{i_1 \dots i_r} = R_{i_1 j_1} \cdots R_{i_r j_r}\, T_{j_1 \dots j_r}$.

A \emph{spherical tensor} of type $l \in \{0,1,2,\dots\}$ is an element of the irreducible $SO(3)$ representation space $V_l$ of dimension $2l+1$. The spherical tensor $T$ with type $l$ has components $T_{m}$ and under a rotation $R$ these components transform according to
$\left[R \cdot T \right]_{m} = \sum_{m'=-l}^l [D^{(l)}(R)]_{m m'}\, T_{m'}$,
where $D^{(l)}(R)$ is the Wigner $D$-matrix of degree $l$. Spherical tensors are the basic building blocks of irreducible representations of $SO(3)$.

\section{Tensor network generators}\label{tns_ge}
In this section, we characterize invariant polynomials utilizing specific symmetric tensors. 

The polynomials of $n$ variables input ${\bm x}=(\bm x_1,\dots, \bm x_n)$, denoted by $F[V]$, constitute an algebra. The group invariant polynomials are
\begin{equation}
    F[V]^G =\{f\in F[V] \mid f(g\cdot \bm x)=f(\bm x)\ \text{for } g\in G\}.
\end{equation}
It's easy to see that the invariant polynomials form a subalgebra denoted by $F[V]^G$. The generators of $F[V]^G$ are a set of invariant polynomials $\{g_1,\dots,g_m\}$ such that any $f\in F[V]^G$ can be written as $f=q(g_1,\dots,g_m)$ for some polynomial $q$. Moreover, the Stone-Weierstrass theorem states that any continuous functions can be approximated by polynomials. It follows, therefore, that any invariant function can be approximated by a function of these generators $\{g_1,\dots,g_m\}$. We'll give a systematic method to construct the generators of $\mathbb R[V]^{SO(3)}$ by symmetric tensor network, which we call \textit{tensor network generators}.

\subsection{Vector inputs}
Firstly, let's consider the simplest case, where the inputs are 3D vectors $\bm x_1,\dots,\bm x_n \in \mathbb{R}^{3}$. In this case, the input space is $V=\mathbb{R}^{3n}$. Weyl \citep{weyl1946classical} proved that 
\begin{lemma}
Let inputs $\bm x_1,\dots,\bm x_n \in \mathbb{R}^{3}$. The set of generators of $\mathbb R[V]^{SO(3)}$ is $\{\bm x_i\cdot \bm x_j,(\bm x_i\times \bm x_j)\cdot \bm x_k\}$.
    Therefore any invariant polynomial takes the form of
$f(\bm x)=q(\{ \bm x_i\cdot \bm x_j,(\bm x_i\times \bm x_j )\cdot \bm x_k \})$.
\end{lemma}\label{thm:vector}
This lemma not only provides a finite set of generators of 3D vector inputs, but can also be used to greatly simplify the structure of a symmetric tensor, which is useful in treating the other  input form that we discuss later.

\begin{lemma}
    Each $SO(3)$ symmetric tensor $T \in (\mathbb{R}^3)^{\otimes n} $ can be constructed by identity tensor $\delta_{ij}$ and Levi-Civita tensor $\epsilon_{ijk}$. That is to say, each $SO(3)$ symmetric tensor $T$ is a linear combination of tensor networks, each of which is the tensor product of $\delta_{ij}$ and $\epsilon_{ijk}$. 
    
    Furthermore, if the rank of $T$ is even, then $T$ is a linear combination of tensor networks, each of which is the tensor product of $\delta_{ij}$. Otherwise, $T$ is a linear combination of tensor networks, each of which is the tensor product of $\delta_{ij}$ together with exactly one $\epsilon_{ijk}$.
    \label{thm:sym_ten}
\end{lemma}
We give the proof in Appendix
\ref{sec:proof_3.2}. The symmetric tensor $\delta_{ij}$ and $\epsilon_{ijk}$ can be represented by
\begin{equation}
\adjustbox{valign=c}{
    \begin{tikzpicture}[baseline={([yshift=-0.5ex]current bounding box.center)}] 
        \node[anchor=east] at (0, 0.25) {$\delta_{ij} :$};
        \draw (0.2, 0) arc (180:0:0.5cm);
        \begin{scope}[xshift=3.0cm] 
            \node[anchor=east] at (0, 0.25) {$\epsilon_{ijk} :$};
            \draw (0.2, 0) arc (180:0:0.5cm);
            \coordinate (bottom_center) at (0.7, 0);  
            \coordinate (top_center) at (0.7, 0.5);    
            \draw (bottom_center) -- (top_center);     
            \fill (top_center) circle (2pt);           
        \end{scope}       
    \end{tikzpicture}
}
\end{equation}

It should be noted that these $SO(3)$ symmetric tensors whose indices take the 3D representation, along with their characteristic properties, are also known as isotropic tensors \citep{Jeffreys_1973}  in the classical invariance theory. 

\subsection{Cartesian tensor inputs} \label{Cartesian_ten}
Next, let's consider a more general case, where the input $\bm x_1,\dots,\bm x_n$ are Cartesian tensors with various ranks $r_i$, i.e. $\bm x_i \in (\mathbb{R}^3)^{\otimes r_i}$. For this case, we can construct the tensor network generators in the following way, 

\begin{proposition} \label{thm:tensor}
    Let $\bm x_1,\dots,\bm x_n$ be input Cartesian tensors with $\bm x_i \in (\mathbb{R}^3)^{\otimes r_i}$, and let $V$ denote the input space. Then $\mathbb R[V]^{SO(3)}$ is generated by the contraction of connected tensor network formed by $\bm x_1,\dots,\bm x_n$ (multiplicity is allowed), together with at most one Levi-Civita tensor $\epsilon_{ijk}$. 
\end{proposition}
The proof of Proposition \ref{thm:tensor} is given in Appendix \ref{sec:proof_3.3}. When all inputs $\bm x_1,\dots,\bm x_n$ are vectors, the tensor network generators are $\{\bm x_i\cdot \bm x_j,( \bm x_i\times  \bm x_j)\cdot \bm x_k\} $, which is the same as in Lemma \ref{thm:vector}. If the inputs $\bm x_1,\dots,\bm x_n$ contain tensors of higher rank, the set of possible tensor network generators may not be finite. By Hilbert's finiteness theorem \citep{hilbert1890ueber,hilbert1893}, there exists a finite generator set $(g_1,\dots,g_m)$, which can be obtained from a finite subset of our tensor network generators. That means only a finite subset of tensor network generators suffices to generate the whole $\mathbb R[V]^{SO(3)}$. However, it's very difficult to determine the exact subset. In practice, we can take the subset of tensor network polynomials with finite degree, which gives a finite-degree approximation to the minimal set of generators. We also noted that recent work \citep{gregory2024learningequivarianttensorfunctions} constructs the equivariant polynomial function with Cartesian tensor inputs in a similar way, which uses the group averaging property of the orthogonal groups \citep{Jeffreys_1973}. 

\subsection{Spherical tensor inputs}
In this section, we consider inputs $\bm x_1,\dots,\bm x_n$ of irreducible representation spaces. Irreducible representations of the $SO(3)$ group can be labeled by a non-negative integer  $l$. The representation $l$ is $(2l+1)$-dimensional, and the common 3D representation is just the representation $1$. Let each $\bm x_i\in V_i$ where $V_i$ is a real linear space of an irreducible $SO(3)$ representation $l_i$. In the following, we describe how to embed each spherical tensor $\bm x_i$ to space $(1)^{\otimes l_i}$

A Cartesian tensor of rank-$l$ whose indices take the 3D representation of $SO(3)$ is a reducible representation and can be reduced as follows.
\begin{equation}
    (1)^{\otimes l}=(l)\oplus(l-1)^{d_{l,l-1}}\oplus\cdots\oplus(j)^{d_{l,j}}\oplus\cdots\oplus(0)^{d_{l,0}}
\end{equation}
where $(1)^{\otimes l}$ denotes the $l$-fold tensor-product representation space i.e. $(1)^{\otimes l} = \underbrace{(1)\otimes\cdots\otimes(1)}_{l}$ and $(j)^{ d_{l,j}} = \underbrace{(j)\oplus\cdots\oplus(j)}_{d_{l,j}}$. The value of $d_{l,j}$ is given in Appendix \ref{sec:product_dec}.

For each representation $l$, we can define a tensor $P_l$ which maps from the  space $(1)^{\otimes l}$ to the irreducible
representation space $l$. An explicit construction of $P_l$ can be obtained by applying the lowering operator to the highest weight vector in $( 1)^{\otimes l}$; see Appendix \ref{sec:projector_form} for details. The tensor diagram of the tensor $P_l$ is illustrated as follows:
\begin{equation}
    \adjustbox{valign=c}{
\begin{tikzpicture}[line cap=round,line join=round]
  \node[draw,circle,minimum size=10mm] (P) at (0,0) {$P_l$};
  \draw (-1.5,0) -- node[above] {$l$} (P.west);
  \coordinate (Rtop) at (2.2, 0.85);
  \coordinate (Rmid) at (2.2, 0.35);
  \coordinate (Rbot) at (2.2,-0.85);
  \draw (P.25)  to[out=25,in=180]  (Rtop);
  \draw (P.0)   to[out=10,in=180]  (Rmid);
  \draw (P.-25) to[out=-25,in=180] (Rbot);
  \node[above] at ($(Rtop)+(0,0.05)$) {$1$};
  \node[below] at ($(Rbot)+(0,-0.05)$) {$1$};
  \node at (1.55,0) {$\vdots$};
  \node[anchor=west] at (2.45,0) {$\left.\vphantom{\rule{0pt}{1.5cm}}\right\}$};
  \node[anchor=west] at (2.85,0) {$l$};
\end{tikzpicture}
}
\end{equation}
where the numbers on the edges represent the irreducible representation type of the space associated with that index. 

\begin{lemma} \label{P_l_proper}
    The tensor $P_l$ have the following properties:
\begin{enumerate}
     \item $P_l$ are $SO(3)$ symmetric tensors.

     \item  $P_l$ are isometry, i.e. $P_l P_l^\dagger = I$ 

     \item The $l$ indices corresponding to space $(1)^{\otimes l}$ of $P_l$ are traceless, i.e., $[P_l]^{m}_{ \dots, a_k, \dots a_k,\dots} =0$ \footnote{We adopt the Einstein summation convention: repeated indices are summed over their full range.}

     \item Contracting any two of the $l$ indices corresponding to space$(1)^{\otimes l}$ of $P_l$ with the Levi–Civita tensor yields zero, i.e., $[P_l]^{m}_{ \dots, a_i, \dots a_j,\dots} \epsilon_{a_ia_jk}=0$

     \item The  $l$ indices corresponding to space $(1)^{\otimes l}$ of $P_l$ is permutation invariant, i.e., $[P_l]^{m}_{ \dots, a_i, \dots a_j,\dots} =[P_l]^{m}_{ \dots, a_j, \dots a_i,\dots}$
\end{enumerate}
\end{lemma}
We give the proof in Appendix \ref{sec:projector_form}.

Let $\bm x$ be a variable of representation $l$, we define $P_l(\bm x)$ to be
\begin{equation}
\adjustbox{valign=c}{
        \begin{tikzpicture}
            \draw (-1,0) node[draw,circle] (P1) {$P_l(\bm x)$};
            \draw (P1) to[out=45,in=180] (0,0.7);
            \draw (P1) to[out=30,in=180] (0,0.4);
            \draw (P1) to[out=-45,in=-180] (0,-0.7);
            \draw (0,0) node{$\vdots$};
        \end{tikzpicture}
        }=
    \adjustbox{valign=c}{
        \begin{tikzpicture}
            \draw (-2,0) node[draw,circle] (x) {$\bm x$};
            \draw (-1,0) node[draw,circle] (P1) {$P_l$};
            \draw (P1) to[out=45,in=180] (0,0.7);
            \draw (P1) to[out=30,in=180] (0,0.4);
            \draw (P1) to[out=-45,in=-180] (0,-0.7);
            \draw (x) -- (P1);
            \draw (0,0) node{$\vdots$};
        \end{tikzpicture}
        }.
\end{equation}
Then $P_l(x)$ becomes a tensor in the space $(1)^{\otimes l}$, which can be used to construct tensor network generators.

\begin{proposition}\label{Spher_gene}
    Let $\bm x_1,\dots,\bm x_n$ be input variables that take the irreducible representation $l_1,\dots,l_n$ of SO(3). Let $V$ be the input space. Then $\mathbb R[V]^{SO(3)}$ is generated by the contraction of connected tensor network formed by $P_{l_1}(\bm x_1),\dots,P_{l_n}(\bm x_n)$ (multiplicity is allowed), together with at most one Levi-Civita tensor $\epsilon_{ijk}$. \label{thm:sub_tensor}
\end{proposition}
The proof is given in Appendix \ref{sec:proof_3.4}.

\subsection{Constructing continuous invariant functions}
Using the tensor network generator we constructed and Stone-Weierstrass Theorem, we can always express the continuous invariant functions in following form:
\begin{proposition}
   Let $\bm x_1,\dots,\bm x_n$ be input variables in space $V_1,\dots,V_n$, any continuous SO(3) invariant function $f :\bigoplus_j V_j\rightarrow \mathbb R$ can be approximately expressed as \footnote{The notion of approximation used this paper is uniform approximation on compact sets: for any compact set $K\subset \bigoplus_{j} V_j$ and any $\eta>0$, we have  $\sup_{x\in K}\big|f(\bm x_1,\dots,\bm x_n)-\tilde{f} (\bm x_1,\dots,\bm x_n)\big|<\eta$.}
   \begin{equation}
       \tilde{f} (\bm x_1,\dots,\bm x_n) =q(g_1,\dots,g_k),
   \end{equation}
   where $g_1,\dots,g_k \in \mathbb R[V]^{SO(3)}$ are tensor network generators of $\bm x_1,\dots ,\bm x_n$, and $q$ is the ordinary polynomial.
   \label{cont_invar_fun}
\end{proposition}
See Appendix \ref{sec:proof_cont_invar_fun} for proof. Since each $g_i\in\mathbb{R}[V]^{SO(3)}$ is exactly invariant, the composed approximant $\tilde f(\bm x_1,\dots,\bm x_n)$ is itself exactly $SO(3)$ invariant (not approximately invariant). In practice, we can use a general neural network that does not incorporate $SO(3)$ symmetry to approximate $q$. Therefore, any continuous $SO(3)$ invariant functions can be approximated by composing such a network with tensor-network generators; see Fig. \ref{fig:inv_ansatz}(a).

\begin{figure}[tb!]
    \centering
    \includegraphics[width=0.7\linewidth]{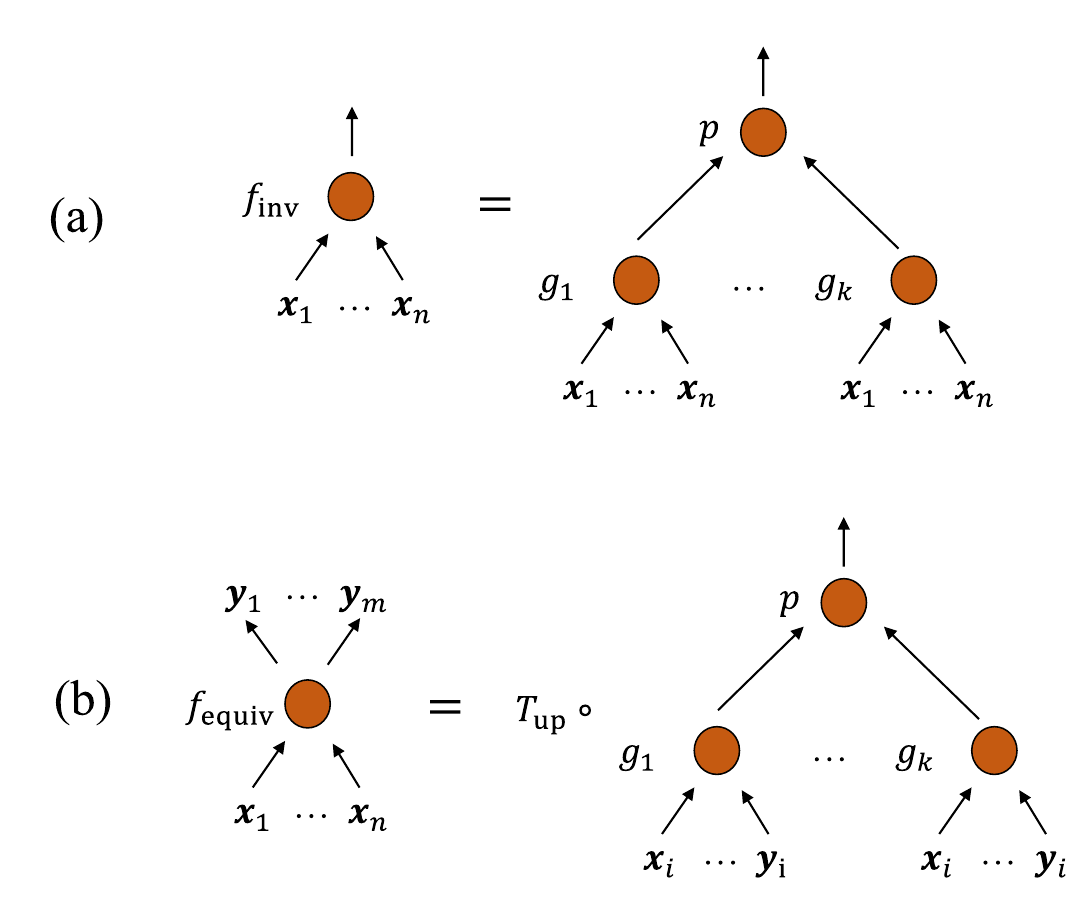}
    \caption{(a) The general form of continuous $SO(3)$ invariant  functions $f_{\rm inv}$ by composing a general neural network $p$ with $g_1,\dots,g_k$, which are tensor network generators. (b) The general form of $SO(3)$ equivariant functions by composing a general neural network $p$ with $g_1,\dots,g_k$, which are tensor network generators, and then applying $T_{\rm up}$.}
    \label{fig:inv_ansatz}
\end{figure}

\section{Generalization to equivariant functions}
Having established a general construction for $SO(3)$ invariant functions, we now extend the framework to  $SO(3)$ equivariant mappings. In fact, we can always construct an equivariant function from an invariant function \citep{blum2023machine}. We have 
\begin{lemma}\label{th_equ_f_in}
Given an invariant function $f: \bigoplus_j V_j \oplus\bigoplus_i U_i \rightarrow F$ with input $\bm x_1,\dots,\bm x_n,\bm y_1,\dots,\bm y_m$ in space $V_1,\dots,V_n,U_1,\dots,U_m$, we can always construct an equivariant function $T_{\rm up}(f):\bigoplus_j V_j\rightarrow \bigoplus_i \bar U_i$, where $G$ acts on $\bar U_i$ by the dual representation of $U_i$, by defining
\begin{equation}
    T_{\rm up}(f)^{i}(\bm x_1, \cdots, \bm x_n)  = \left. \frac{\partial f(\bm x_1, \cdots, \bm x_n, \bm y_1,\dots,\bm y_m)}{\partial \bm y_i} \right|_{\bm y_1=\cdots=\bm y_m=0}    
\end{equation}\label{eqn:inv2equiv}
where we have chosen a natural set of basis for $U_i$ and the corresponding dual basis for $\bar U_i$. Moreover, any equivariant function can be obtained in this way.
\end{lemma}
We give the proof of the above Lemma and more details about obtaining equivariant functions from the invariant functions in Appendix \ref{sec:equ_from_in}. 

From the Proposition \ref{cont_invar_fun} and Lemma \ref{th_equ_f_in}, we can construct equivariant functions from invariant functions constructed by tensor network generators, as shown in Fig. \ref{fig:inv_ansatz}(b).
We can obtain the following proposition (We give the proof of the proposition in Appendix \ref{proof_equ_tns}.),
\begin{proposition} \label{equ_tns}
    Let $\bm x_1,\dots,\bm x_n$ be input variables in space $V_1,\dots,V_n$ and $\bm y_1,\dots,\bm y_m$ be output variables in space $U_1,\dots,U_m$, any continuous $SO(3)$ equivariant function $h :\bigoplus_j V_j\rightarrow \bigoplus_i  U_i$ can be approximately expressed as
    \begin{equation}\label{eq:equi_appro}
        \tilde{h}^i(\bm x_1,\dots,\bm x_n) = \sum_j ^{N_i}q^i_{j}(g_1,\dots,g_k) t^i_j(\bm x_1,\dots,\bm x_n)
    \end{equation}
    where $q^i_{j}(\cdot)$ are the ordinary polynomials, $g_1,\dots,g_k \in \mathbb R[V]^{SO(3)}$ are tensor network generators of $\bm x_1,\dots ,\bm x_n$, and $t_j^i(\bm x_1,\dots,\bm x_n)$ are tensor networks (labeled by $j$) which are obtained by removing the output variables $\bm y_i$ from the tensor network generators of $\bm x_1,\dots ,\bm x_n,  \bm y_i$, in which the $\bm y_i$ variable appears exactly once.
\end{proposition}
The composed approximant $\tilde h(\bm x_1,\dots,\bm x_n)$ is itself exactly $SO(3)$ equivariant (not approximately equivariant). Similarly, $q^i_{j}(\cdot)$ can be parameterized by a generic neural network. In this way, we decouple the symmetry constraints from the neural network architecture. 

\section{Equivariant functions with vector inputs}
In practice, a common learning scenario involves using the 3D coordinates of a set of particles (atoms) as inputs to predict their physical properties. These targets are not only limited to scalars; they often possess geometric significance and must transform consistently under rotation. Typical examples include the polarizability tensor and stress tensor, which are naturally represented as Cartesian tensors, as well as quantities from electronic-structure such as Hamiltonians and density matrices, which can be represented using spherical-tensor features. Motivated by these applications, this section focuses on equivariant maps with vector inputs and Cartesian or spherical tensor outputs, and derives the general form of these maps based on the framework we develop above.

Actually, for equivariant functions with vector inputs, the family of tensor networks $t_j$ in Eq. (\ref{eq:equi_appro}) required to represent general continuous equivariant maps can always be chosen to be finite, since there are only finitely many connected tensor networks built from the input vectors and fixed output tensors. 
\begin{lemma}
    Let $\bm x_1,\ldots,\bm x_n$ be input variables in spaces $V_1,\ldots,V_n$ and $\bm y_1,\ldots,\bm y_m$ be output variables in spaces $U_1,\ldots,U_m$, where each $V_j = \mathbb{R}^3$ and each $U_i$ is either a Cartesian tensor space or a spherical tensor space. Any continuous $SO(3)$ equivariant function $h :\bigoplus_j V_j\rightarrow \bigoplus_i U_i$ can be approximated by Eq. (\ref{eq:equi_appro}) with finite integer $N_i$  of tensor network terms for each output channel $i\in{1,\ldots,m}$.
\end{lemma}

Propositions \ref{Cartesian_eq} and \ref{spher_eq} below provide explicit forms of equivariant functions with a single Cartesian tensor output and functions with a single spherical tensor output, respectively.
\begin{proposition}\label{Cartesian_eq}
   Any continuous $SO(3)$ equivariant function $h:(\mathbb{R}^3)^n \to (\mathbb{R}^3)^{\otimes r}$ with vector inputs $\bm x_1,\dots,\bm x_n$ and  Cartesian tensor output of rank $r$ can  be approximately expressed as
\begin{equation}
\begin{split}
     \tilde{h}(\bm x_1,\dots,\bm x_n) =
     \sum_i^{N} q_i(g_1,\dots,g_m)t_i(\bm x_1,\dots,\bm x_n)
\end{split}
\end{equation} 
where $q_i$ are ordinary polynomials,   
$g_1,\dots,g_m \in \{\bm x_i\cdot \bm x_j,(\bm x_i\times \bm x_j)\cdot \bm x_k\} $ and $t_i(\bm x_1,\dots,\bm x_n)$ are tensor networks with $r$ open legs consisting of $\bm x_1,\dots,\bm x_n$ (multiplicity is allowed), $\delta_{ij}$ and at most one Levi-Civita tensor $\epsilon_{ijk}$.
\end{proposition}

Thanks to the nice properties of tensor $P_l$ of lemma \ref{P_l_proper}, we can get more concise form of equivariant functions with a single spherical tensor output; see the proof of Proposition \ref{spher_eq} in Appendix \ref{Proof_spher_eq}.
\begin{proposition}\label{spher_eq}
Any continuous $SO(3)$ equivariant function $h:(\mathbb{R}^3)^n \to \mathbb{R}^{2l+1}$ with vector inputs $\bm x_1,\dots,\bm x_n$ and  spherical tensor output of type $l$ can  be approximately expressed as
\begin{equation}
\begin{split}
     \tilde{h}(\bm x_1,\dots,\bm x_n) =
     \sum_i^{N} q_i(g_1,\dots,g_m)t_i(\bm x_1,\dots,\bm x_n)
\end{split}
\end{equation}
where $q_i$ are ordinary polynomials,   
$g_1,\dots,g_m \in \{\bm x_i\cdot \bm x_j,(\bm x_i\times \bm x_j)\cdot \bm x_k\} $ and  $t_i(\bm x_1,\dots,\bm x_n)$ are connected tensor network consisting of $\bm x_1,\dots,\bm x_n$, a single $P_l$ and at most one Levi-Civita tensor $\epsilon_{ijk}$ in following form,
\begin{equation}\label{eq:connect_spher}
\raisebox{0.4ex}{
\adjustbox{valign=c}{
\begin{tikzpicture}[
  xnode/.style={draw,circle,inner sep=1pt,minimum size=4mm},
  pnode/.style={draw,circle,inner sep=1pt,minimum size=7mm},
  jdot/.style={circle,fill,inner sep=0pt,minimum size=1.2mm}
]
  \def\dy{0.85}
  \begin{scope}
    \node[pnode] (P1) at (0,0) {$P_l$};
    \draw (-0.85,0) -- (P1.west);
    \def\xL{1.35} 
    \node[xnode] (Ltop) at (\xL,  1.35*\dy) {};
    \node[xnode] (Lmid) at (\xL,  0)       {};
    \node[xnode] (Lbot) at (\xL, -1.35*\dy) {};
    \node at (\xL,  0.65*\dy) {$\vdots$};
    \node at (\xL, -0.65*\dy) {$\vdots$};
    \draw (-0.85,0) -- (P1.180);
    \draw (P1.35)  -- (Ltop.west);
    \draw (P1.0)   -- (Lmid.west);
    \draw (P1.-35) -- (Lbot.west);
  \end{scope}
  \node at (2.25,0) {\textit{or}};

\begin{scope}[shift={(3.75,0)}]
  \node[pnode] (P2) at (0,0) {$P_l$};
  \draw (-0.85,0) -- (P2.180);
  \def\xR{1.35}
  \def\xJ{0.70}
  \node[xnode] (R1) at (\xR,  1.35*\dy) {};
  \node[xnode] (R2) at (\xR,  0.45*\dy) {};
  \node[xnode] (R3) at (\xR, -0.45*\dy) {};
  \node[xnode] (R4) at (\xR, -1.35*\dy) {};
  \node at (\xR,  0.96*\dy) {$\vdots$};
  \node at (\xR, -0.85*\dy) {$\vdots$};
  \node[jdot] (J) at (\xJ,0) {};
  \draw (P2.35)  -- (R1.west);
  \draw (P2.0)   -- (J);     
  \draw (P2.-35) -- (R4.west);  
  \draw (J) -- (R2.west);
  \draw (J) -- (R3.west);
\end{scope}

\end{tikzpicture}}}
\end{equation}

where $\adjustbox{valign=c}{
\begin{tikzpicture}
  \node[draw,circle,inner sep=1pt,minimum size=4mm] (c) at (0,0) {};
  \draw (-.5,0) -- (c.west);
\end{tikzpicture}
} \in \{\bm x_1,\dots,\bm x_n\}$ and multiplicity is allowed.
\end{proposition}

It is also straightforward to extend the above construction to the $O(3)$ equivariant setting.
Since $O(3) \cong SO(3) \times I$, where $I = \{e, i\}$ is the inversion group. The $O(3)$ equivariant functions can be obtained by averaging $SO(3)$ equivariant functions $\tilde h$ over the two elements of $I$. Therefore, we have the following corollary.
\begin{corollary}
    Consider a continuous $O(3)$ equivariant function $h:(\mathbb{R}^3)^n \to V$ with vector inputs $\bm x_1,\dots,\bm x_n$, where the output $V$ is either a single Cartesian tensor of rank $r$, or a single spherical tensor of type $l$ and in both cases has parity $p\in \{+1,-1\}$ \footnote{True tensors have $p=+1$, pseudotensors have $p=-1$}. Then $h$ can be approximated by
\begin{equation}
        \tilde h^{O(3)}(\bm x_1,\ldots,\bm x_n)
= \frac12\Big(\tilde h(\bm x_1,\dots,\bm x_n)+\rho \tilde h(-\bm x_1,\ldots,-\bm x_n)\Big),
\end{equation}
where
\begin{equation}
\rho = p(-1)^{k},
k=
\begin{cases}
r, & \text{ Cartesian tensor},\\
l, & \text{ spherical tensor}.
\end{cases}
\end{equation}
\end{corollary}

\section{Constructing equivariant operations for geometric graph neural network}
In this section, we will show how common equivariant primitives in
geometric GNNs arise within our framework and how to use the framework to design equivariant operations.

Geometric GNNs are built on graph-structured data with the 3D geometric information. For the node feature $\mathbf{h}_{i}$  of node $i$ and $\mathbf{h}_{j}$ of its neighbour $j \in\ N(i)$ in layer $l$, the interaction message $\mathbf{m}_{ij}^{l+1}$ for them is $\mathbf{m}_{ij}^{l+1}=f_{m}\left(\mathbf{h}_{i}^{l},\mathbf{h}_{j}^{l}\right)$, where the map $f_{m}$ is a learnable function. Then the interaction message $\mathbf{m}_{ij}$ for all the neighbour $j \in\ N(i)$ are aggregated by a permutation invariant function $\bigoplus_{j \in \mathcal{N}(i)}$, such as sum and mean, which is used to update the node feature $\mathbf{h}_{i}$  of node $i$ according to another learnable function  $\mathbf{h}_{i}^{l+1}=f_{u}\left(\mathbf{h}_{i}^{l}, \bigoplus_{j \in \mathcal{N}(i)} \mathbf{m}_{i j}^{l}\right)$. For the equivariant geometric GNNs, the functions $f_{m}$ and $f_{u}$ need to be $SO(3)$ equivariant functions. Many studies on equivariant geometric GNNs focus on designing novel equivariant functions $f_{m}$ and $f_{u}$. The feature $\mathbf{h}_{i}$ and message $\mathbf{m}_{ij}$ can be cartesian tensor \citep{Wang2024} or spherical tensor \citep{thomas2018tensorfieldnetworksrotation,NEURIPS2018_488e4104}.

In the following, we show how to construct the functions $f_{m}( \mathbf{h}_i, \mathbf{h}_j)$. In the diagrams below, nodes represent the input/output quantities, and connecting lines represent the contraction of indices with the identity tensor $\delta_{ij}$ or the Levi-Civita tensor $\epsilon_{ijk}$. Specifically, the red lines denote the output space corresponding to the indices of the message $\mathbf{m}_{ij}$ before the final contraction to a scalar. By taking the derivative with respect to $\mathbf{m}_{ij}$ (removing the red lines), we obtain the equivariant functions. For simplicity, we restrict our attention to linear functions—specifically, tensor networks in which $\mathbf{h}_i$ and $\mathbf{h}_j$ appear at most once. These represent the most common equivariant operations in geometric GNNs. More complicated functions can be constructed in the similar way.

\subsection{Constructing equivariant operations for Cartesian tensor feature}
For Cartesian tensor feature $\mathbf{h}$ and message $\mathbf{m}_{ij}$, we consider the tensor rank to be 1 (vector) and 2 (matrix). 

\textbf{Vector input and vector output:}
We set $\mathbf{h}_i = \mathbf{u}$, $\mathbf{h}_j = \mathbf{v}$, $\mathbf{m}_{ij} = \mathbf{w}$ and $\mathbf{u}, \mathbf{v}, \mathbf{w}\in \mathbb{R} ^{3} $. According to Proposition \ref{Cartesian_eq}, the equivariant polynomials $t(\bm u , \bm v)$ are tensor networks with 1 open legs consist of $\bm u , \bm v$, $\delta_{ij}$ and at most one $\epsilon_{ijk}$. So the possible diagrams are
\begin{equation}
    \adjustbox{valign=c}{
    \begin{tikzpicture}
    [yscale=0.7,xscale=0.7]
        \draw (0,0) node[draw,circle, minimum size=0.4cm] (n1) {};
        \draw (n1) -- ++(0,0.8) arc (180:0:0.6);
        
    \end{tikzpicture}
}
\qquad
\adjustbox{valign=c}{
    \begin{tikzpicture}
    [yscale=0.7,xscale=0.7]
        \draw (-1.2,0) node[draw,circle, minimum size=0.4cm] (n1) {};
        \draw (0,0)    node[draw,circle, minimum size=0.4cm] (n2) {};
        
        \draw (n1) -- ++(0,0.8) arc (180:0:1.2 and 0.6);
        
        \draw (n2) -- (0, 1.4);
        \fill (0, 1.4) circle (2pt);
    \end{tikzpicture}
}
\end{equation}
The first diagram is $\mathbf{u}$ or $\mathbf{v}$ and the second diagram is  $\mathbf{u}\times  \mathbf{v}$. These operations are typically the equivariant operations used in the EGNN-style \cite{pmlr-v139-satorras21a,pmlr-v139-schutt21a,9711441,le2022equivariantgraphattentionnetworks}. 

\textbf{Matrix input and matrix output:}
For matrix features, we can set $\mathbf{h}_i = \mathbf{A}$, $\mathbf{h}_j = \mathbf{B}$, $\mathbf{m}_{ij} = \mathbf{C}$ and $\mathbf{A}, \mathbf{B}, \mathbf{C} \in \mathbb{R} ^{3 \times 3}$. According to Proposition \ref{equ_tns}, we first construct the  tensor network generators of $\mathbf{A}, \mathbf{B}, \mathbf{C}$ that output $\mathbf{C}$ appears exactly once. As we stated above, we just consider linear maps for simplicity.  So the possible diagrams are
\begin{equation}
    \tikzset{
        block/.style={draw,   rounded corners=2pt, minimum width=0.6cm, minimum height=0.2cm}
    }
    \adjustbox{valign=c}{
    \begin{tikzpicture}[yscale=0.7,xscale=0.7]
        \node[block] (n1) at (0,0) {};
        \draw[  red] ($(n1.north)+(-0.2,0)$) -- ++(0,0.6) coordinate (left_top);
        \draw (left_top) arc (180:0:0.2 and 0.1) coordinate (right_top);
        \draw[  red] (right_top) -- ($(n1.north)+(0.2,0)$);
    \end{tikzpicture}
    }
    \qquad
    \adjustbox{valign=c}{
    \begin{tikzpicture}[yscale=0.7,xscale=0.7]
        \node[block] (n1) at (0,0) {};
        \node[block] (n2) at (1.2,0) {};
        \draw  ($(n1.north)+(-0.2,0)$) -- ++(0,0.6) arc (180:0:0.8 and 0.4) coordinate (out_end);
        \draw[  red] (out_end) -- ($(n2.north)+(0.2,0)$);
        \draw  ($(n1.north)+(0.2,0)$) -- ++(0,0.6) arc (180:0:0.4 and 0.2) coordinate (in_end);
        \draw[  red] (in_end) -- ($(n2.north)+(-0.2,0)$);
    \end{tikzpicture}
    }
    \qquad
    \adjustbox{valign=c}{
    \begin{tikzpicture}[yscale=0.7,xscale=0.7]
        \node[block] (n1) at (0,0) {};
        \node[block] (n2) at (1.2,0) {};
        \node[block] (n3) at (2.4,0) {};
        \draw  ($(n1.north)+(-0.2,0)$) -- ++(0,0.8) arc (180:0:1.4 and 0.7) coordinate (big_end);
        \draw[  red] (big_end) -- ($(n3.north)+(0.2,0)$);
        \draw  ($(n1.north)+(0.2,0)$) -- ++(0,0.6) arc (180:0:0.4 and 0.2) coordinate (mid1);
        \draw  (mid1) -- ($(n2.north)+(-0.2,0)$);       
        \draw  ($(n2.north)+(0.2,0)$) -- ++(0,0.6) arc (180:0:0.4 and 0.2) coordinate (mid2);
        \draw[  red] (mid2) -- ($(n3.north)+(-0.2,0)$);
    \end{tikzpicture}
    }
\end{equation}
which correspond to $\text{Tr}(\mathbf{C})$, \{$\text{Tr}(\mathbf{A}\mathbf{C})$, $\text{Tr}(\mathbf{B}\mathbf{C})\}, \{\text{Tr}(\mathbf{A}\mathbf{B}\mathbf{C})$, $\text{Tr}(\mathbf{B}\mathbf{A}\mathbf{C})$\}.\footnote{In each trace expression, any occurrence of $\mathbf{A}$, $\mathbf{B}$, or $\mathbf{C}$ may be replaced by its transpose, yielding additional valid generators, e.g. $\mathrm{tr}(\mathbf{A}^\top \mathbf{C})$, $\mathrm{tr}(\mathbf{A}\mathbf{C}^\top)$, $\mathrm{tr}(\mathbf{A}^\top \mathbf{B} \mathbf{C}^\top)$, etc.}

We can get equivariances by removing the output,
\begin{equation}
    \tikzset{
        block/.style={draw,   rounded corners=2pt, minimum width=0.6cm, minimum height=0.2cm}
    }
    \adjustbox{valign=c}{
    \begin{tikzpicture}[yscale=0.7,xscale=0.7]
        \draw  ($(n1.north)+(-0.3,0)$) arc (180:0:0.3);       
    \end{tikzpicture}
    }
    \qquad
    \adjustbox{valign=c}{
    \begin{tikzpicture}[yscale=0.7,xscale=0.7]
        \node[block] (n1) at (0,0) {};
        \draw  ($(n1.north)+(-0.2,0)$) -- ++(0,0.6) arc (180:0:0.8 and 0.4) coordinate (out_end);
        \draw  ($(n1.north)+(0.2,0)$) -- ++(0,0.6) arc (180:0:0.4 and 0.2) coordinate (in_end);
    \end{tikzpicture}
    }
    \qquad
    \adjustbox{valign=c}{
    \begin{tikzpicture}[yscale=0.7,xscale=0.7]
        \node[block] (n1) at (0,0) {};
        \node[block] (n2) at (1.2,0) {};
        \draw  ($(n1.north)+(-0.2,0)$) -- ++(0,0.8) arc (180:0:1.4 and 0.7) coordinate (big_end);
        \draw  ($(n1.north)+(0.2,0)$) -- ++(0,0.6) arc (180:0:0.4 and 0.2) coordinate (mid1);
        \draw  (mid1) -- ($(n2.north)+(-0.2,0)$);      
        \draw  ($(n2.north)+(0.2,0)$) -- ++(0,0.6) arc (180:0:0.4 and 0.2) coordinate (mid2);
    \end{tikzpicture}
    }
\end{equation}
which is $\mathbf{I}$, $\{ \mathbf{A},\mathbf{B} \}$, $ \{\mathbf{A}\mathbf{B}, \mathbf{B}\mathbf{A} \}$. These equivariant operations correspond to tensor contraction and summation in the context of Cartesian tensor features \citep{Wang2024}. 

\textbf{Matrix input and vector output:}
If we set the matrix features $\mathbf{h}_i = \mathbf{A}$, $\mathbf{h}_j = \mathbf{B}$ and $\mathbf{A},  \mathbf{B} \in \mathbb{R} ^{3 \times 3}$ and the interaction message $\mathbf{m}_{ij} = \mathbf{v}$ $\in \mathbb{R} ^{3}$. The invariances we can construct are as follows,
\begin{equation}
\tikzset{
    block/.style={draw,   rounded corners=2pt, minimum width=0.6cm, minimum height=0.2cm},
    ball/.style={draw,   circle, minimum size=0.4cm},
    dot/.style={fill, circle, minimum size=3pt, inner sep=0pt},
    wire/.style={black},
    redwire/.style={red}
}
    \adjustbox{valign=c}{
    \begin{tikzpicture}[scale=0.8]
        \node[block] (n1) at (0,0) {};
        \node[ball]  (c1) at (1.3,0) {}; 
        \coordinate (Apex) at ($(n1.north)+(0.25, 0.8)$);
        \draw[wire] ($(n1.north)+(0.25,0)$) -- (Apex) node[dot] {};
        \draw[wire] (Apex) arc (90:180:0.5) -- ($(n1.north)+(-0.25,0)$);
        \draw[wire] (Apex) arc (90:0:1.05 and 0.5) coordinate (end_arc);
        \draw[redwire] (end_arc) -- (c1.north);
    \end{tikzpicture}
    }
    \adjustbox{valign=c}{
    \begin{tikzpicture}[scale=0.8]
        \node[block] (n1) at (0,0) {};
        \node[block] (n2) at (1.3,0) {};
        \node[ball]  (c1) at (2.6,0) {};
        \coordinate (Apex) at ($(n1.north)+(0.25, 0.8)$);
        \draw[wire] ($(n1.north)+(0.25,0)$) -- (Apex) node[dot] {};
        \draw[wire] (Apex) arc (90:180:0.5) -- ($(n1.north)+(-0.25,0)$);
        \draw[wire] (Apex) arc (90:0:0.8 and 0.5) -- ($(n2.north)+(-0.25,0)$);
        \draw[wire] ($(n2.north)+(0.25,0)$) -- ++(0,0.5) arc (180:0:0.525) coordinate (end2_temp);
        \draw[redwire] (end2_temp) -- (c1.north);
    \end{tikzpicture}
    }
    \adjustbox{valign=c}{
    \begin{tikzpicture}[scale=0.8]
        \node[block] (n1) at (0,0) {};
        \node[block] (n2) at (1.3,0) {};
        \node[ball]  (c1) at (2.6,0) {};
        \coordinate (BigApex) at ($(n2.north)+(0.25, 1.0)$);
        \draw[wire] ($(n2.north)+(0.25,0)$) -- (BigApex) node[dot] {};
        \draw[wire] (BigApex) arc (90:180:1.8 and 0.6) -- ($(n1.north)+(-0.25,0)$);
        \draw[wire] (BigApex) arc (90:0:1.05 and 0.6) coordinate (end3);
        \draw[redwire] (end3) -- (c1.north);
        \draw[wire] ($(n1.north)+(0.25,0)$) -- ++(0,0.4) arc (180:0:0.4) -- ($(n2.north)+(-0.25,0)$);

    \end{tikzpicture}
    }
\end{equation}
We can get the corresponding equivariances by derivation, 
\begin{equation}
\tikzset{
    block/.style={draw,   rounded corners=2pt, minimum width=0.6cm, minimum height=0.2cm},
    ball/.style={draw,   circle, minimum size=0.4cm},
    dot/.style={fill, circle, minimum size=3pt, inner sep=0pt},
    wire/.style={  black},
    redwire/.style={  red}
}
    \adjustbox{valign=c}{
    \begin{tikzpicture}[scale=0.8]
        \node[block] (n1) at (0,0) {};
        \coordinate (Apex) at ($(n1.north)+(0.25, 0.8)$);
        \draw[wire] ($(n1.north)+(0.25,0)$) -- (Apex) node[dot] {};
        \draw[wire] (Apex) arc (90:180:0.5) -- ($(n1.north)+(-0.25,0)$);
        \draw[wire] (Apex) arc (90:0:1.05 and 0.5) coordinate (end_arc);
    \end{tikzpicture}
    }
    \adjustbox{valign=c}{
    \begin{tikzpicture}[scale=0.8]
        \node[block] (n1) at (0,0) {};
        \node[block] (n2) at (1.3,0) {};
        \coordinate (Apex) at ($(n1.north)+(0.25, 0.8)$);
        \draw[wire] ($(n1.north)+(0.25,0)$) -- (Apex) node[dot] {};
        \draw[wire] (Apex) arc (90:180:0.5) -- ($(n1.north)+(-0.25,0)$);
        \draw[wire] (Apex) arc (90:0:0.8 and 0.5) -- ($(n2.north)+(-0.25,0)$);
        \draw[wire] ($(n2.north)+(0.25,0)$) -- ++(0,0.5) arc (180:0:0.525) coordinate (end2_temp);
    \end{tikzpicture}
    }
    \adjustbox{valign=c}{
    \begin{tikzpicture}[scale=0.8]
        \node[block] (n1) at (0,0) {};
        \node[block] (n2) at (1.3,0) {};
        \coordinate (BigApex) at ($(n2.north)+(0.25, 1.0)$);
        \draw[wire] ($(n2.north)+(0.25,0)$) -- (BigApex) node[dot] {};
        \draw[wire] (BigApex) arc (90:180:1.8 and 0.6) -- ($(n1.north)+(-0.25,0)$);
        \draw[wire] (BigApex) arc (90:0:1.05 and 0.6) coordinate (end3);
        \draw[wire] ($(n1.north)+(0.25,0)$) -- ++(0,0.4) arc (180:0:0.4) -- ($(n2.north)+(-0.25,0)$);

    \end{tikzpicture}
    }
\end{equation}
where the first diagram $A_{ij} \epsilon_{ijk}$ is the axial vector, which consists of the elements of the antisymmetric part of $A_{ij}$.

\subsection{Constructing equivariant operations for spherical tensor feature}
We can also use this framework to  express the equivariant function for spherical tensor inputs and outputs. We set the equivariant function with inputs feature $\mathbf{h}_i = \bm a$ of representation $l_a$, $\mathbf{h}_j = \bm b$ of representation $l_b$ and output message $\mathbf{m}_{ij} = \bm c$ of representation $l_c$. According to Proposition \ref{Spher_gene} we can construct the tensor network generator
\begin{equation}
    \adjustbox{valign = c, scale=0.8}{\begin{tikzpicture}
        \draw (90:1) node[draw,circle] (pc) {$P_{l_c}$};
        \draw (-30:1) node[draw,circle] (pb) {$P_{l_b}$};
        \draw (210:1) node[draw,circle] (pa) {$P_{l_a}$};
        \draw (90:2.3) node[draw,circle] (c) {$\bm c$};
        \draw (-30:2.3) node[draw,circle] (b) {$\bm b$};
        \draw (210:2.3) node[draw,circle] (a) {$\bm a$};
        \draw (pa) edge["\footnotesize$l_a$"'] (a);
        \draw (pb) edge["\footnotesize$l_b$"] (b);
        \draw (pc) edge[color ={rgb, 255:red, 208; green, 2; blue, 27 },"\footnotesize\textcolor{black}{$l_c$}"] (c);
        \draw (pa) edge[ultra thick] (pb);
        \draw (pb) edge[ultra thick] (pc);
        \draw (pc) edge[ultra thick] (pa);
        \
    \end{tikzpicture}}
    \qquad
    \adjustbox{valign = c, scale=0.8}{\begin{tikzpicture}
        \node[circle, fill, inner sep=1.5pt] (fill) at (0,0) {}; 
        \draw (90:1) node[draw,circle] (pc) {$P_{l_c}$};
        \draw (-30:1) node[draw,circle] (pb) {$P_{l_b}$};
        \draw (210:1) node[draw,circle] (pa) {$P_{l_a}$};
        \draw (90:2.3) node[draw,circle] (c) {$\bm c$};
        \draw (-30:2.3) node[draw,circle] (b) {$\bm b$};
        \draw (210:2.3) node[draw,circle] (a) {$\bm a$};        
        \draw (pa) edge["\footnotesize$l_a$"'] (a);
        \draw (pb) edge["\footnotesize$l_b$"] (b);
        \draw (pc) edge[color ={rgb, 255:red, 208; green, 2; blue, 27 },"\footnotesize\textcolor{black}{$l_c$}"] (c);        
        \draw (pa) edge[ultra thick] (pb);
        \draw (pb) edge[ultra thick] (pc);
        \draw (pc) edge[ultra thick] (pa);
        \draw (fill) edge (pa);
        \draw (fill) edge (pb);
        \draw (fill) edge (pc);
    \end{tikzpicture}}
\end{equation}

The thick line in the above diagrams means a group legs with type-1. See details in Appendix \ref{sec:sph_equiv_example}.

In our framework, we can obtain the equivariant function by removing the output tensor.
\begin{equation}
    \adjustbox{valign = c, scale=0.8}{\begin{tikzpicture}
        \draw (90:1) node[draw,circle] (pc) {$P_{l_c}$};
        \draw (-30:1) node[draw,circle] (pb) {$P_{l_b}$};
        \draw (210:1) node[draw,circle] (pa) {$P_{l_a}$};
        \coordinate (c) at (90:2.3) ;
        \draw (-30:2.3) node[draw,circle] (b) {$\bm b$};
        \draw (210:2.3) node[draw,circle] (a) {$\bm a$};
        \draw (pa) edge["\footnotesize$l_a$"'] (a);
        \draw (pb) edge["\footnotesize$l_b$"] (b);
        \draw (pc) edge["\footnotesize$l_c$"] (c);
        \draw (pa) edge[ultra thick] (pb);
        \draw (pb) edge[ultra thick] (pc);
        \draw (pc) edge[ultra thick] (pa);
        \
    \end{tikzpicture}}
    \qquad
    \adjustbox{valign = c, scale=0.8}{\begin{tikzpicture}
        \node[circle, fill, inner sep=1.5pt] (fill) at (0,0) {}; 
        \draw (90:1) node[draw,circle] (pc) {$P_{l_c}$};
        \draw (-30:1) node[draw,circle] (pb) {$P_{l_b}$};
        \draw (210:1) node[draw,circle] (pa) {$P_{l_a}$};
        \coordinate (c) at (90:2.3) ;
        \draw (-30:2.3) node[draw,circle] (b) {$\bm b$};
        \draw (210:2.3) node[draw,circle] (a) {$\bm a$};        
        \draw (pa) edge["\footnotesize$l_a$"'] (a);
        \draw (pb) edge["\footnotesize$l_b$"] (b);
        \draw (pc) edge["\footnotesize$l_c$"] (c);     
        \draw (pa) edge[ultra thick] (pb);
        \draw (pb) edge[ultra thick] (pc);
        \draw (pc) edge[ultra thick] (pa);
        \draw (fill) edge (pa);
        \draw (fill) edge (pb);
        \draw (fill) edge (pc);
    \end{tikzpicture}}
\end{equation}
These are precisely the TP operations of the TFN-style proposed in works \citep{thomas2018tensorfieldnetworksrotation,NEURIPS2018_488e4104} (we give more detail in Appendix \ref{sec:sph_equiv_example}).

\section{Related work}\label{re_wor}
\textbf{Theory of invariant and equivariant functions}: For Cartesian tensors: work \cite{NEURIPS2021_f1b07759} focus on constructing the $O(n)$ and $SO(n)$ invariant and equivariant functions with vector inputs and outputs; works \cite{gregory2024learningequivarianttensorfunctions} and \cite{10756158}  show how to construct the $O(n)$ invariant and equivariant polynomials of the Cartesian tensors in theory;  work
\cite{pmlr-v202-pearce-crump23a} constructed $O(n)$ and $SO(n)$ equivariant linear maps between Cartesian tensors based on Brauer algebra \cite{brauer1937algebras}. For spherical tensor: 3D Steerable CNNs \citep{NEURIPS2018_488e4104} characterize the equivariant linear maps between the spherical tensor feature. Work \cite{Yarotsky2022} proposes invariant/equivariant maps based on shallow neural networks and the given generating set of invariant and equivariant polynomials, and give the concrete construct for $SE(2)$ symmetric functions. Furthermore, work \cite{blum2023machine} show how to get equivariant polynomials from the derivative of invariant polynomials based on the method of B. Malgrange. 

\textbf{Geometric graph neural networks.} 
Existing works for geometric GNNs can be categorized by their strategy for constructing equivariant operations.
Scalar-based approaches generate features primarily through invariant operations by inner products and the equivariant operations typically by vector summations and product  \citep{pmlr-v139-satorras21a,pmlr-v139-schutt21a,9711441}. Tensor Product-based approaches, pioneered by TFN \citep{thomas2018tensorfieldnetworksrotation} and 3D Steerable CNNs \citep{NEURIPS2018_488e4104}, use the higher-type spherical tensors as feature and construct equivariant operations using Clebsch-Gordan tensor products. This line of work also includes \cite{le2022equivariantgraphattentionnetworks,brandstetter2022geometric,liao2023equiformer}.
Recently, higher-rank Cartesian tensors are also used as the equivariant feature in the message passing \citep{Wang2024}. 

\textbf{Tensor network:} The structured decompositions of high-dimensional tensors into networks of lower-dimensional tensors were originally developed in the context of quantum many-body physics \citep{PhysRevLett.91.147902,PhysRevB.73.094423,SCHOLLWOCK201196}, but have since found widespread applications in machine learning \citep{levine2018deep,NEURIPS2019_2bd2e337,NEURIPS2022_fe91414c,wang2025tensornetworksmeetneural}, quantum computing \citep{PhysRevLett.129.090502,PhysRevLett.128.030501}, applied mathematics \citep{doi:10.1137/090752286}, and beyond. For symmetric tensor networks, early contributions  \citep{PhysRevA.82.050301,PhysRevB.83.115125,PhysRevB.86.195114} described how to incorporate the $SU(2)$ symmetry for tensor networks states for quantum many-body systems. More recently, work \cite{Li_2024} use the fusion diagram (a graphical representation of the successive Clebsch–Gordan products) to construct the $SO(3)$ equivariant blocks. Work \cite{Hodapp_2024} used the symmetric tensor network for constructing the machine-learning interatomic potentials. In addition, work \cite{10756158} shows how to construct the $O(n)$ invariant and equivariant polynomials of Cartesian tensors by tensor network. 

In this work, we establish a general framework for constructing concrete tensor network generators and characterize continuous $SO(3)$ invariant and equivariant functions, capable of handling inputs and outputs composed of a tuple of Cartesian and spherical tensors.

\section{Conclusion and Discussion}
This work presents a unified and constructive framework for characterizing general continuous $SO(3)$ symmetric functions. It covers both linear and nonlinear maps and supports Cartesian and spherical tensor inputs and outputs. We introduce the tensor network generators: invariants that consist of connected tensor network templates built from the inputs together with some specific symmetric tensors i.e. $\delta,\epsilon,P_l$, while equivariants follow from the similar templates with open legs. Then we construct the general continuous invariant and equivariant functions built on these tensor networks. Moreover, for the most prevalent setting in scientific modeling—vector inputs and Cartesian or spherical tensor outputs—we derive a concrete and concise form for general continuous $SO(3)$ equivariant maps, and extend it to $O(3)$ by explicitly accounting for parity.

From a practical perspective, the framework can be plugged into machine learning tasks that require rotational symmetry: one first constructs a bank of invariant/equivariant quantities by tensor network templates, and then applies conventional neural networks to learn the remaining parameters. It is also useful for deriving symmetry-preserving operations inside geometric graph neural networks. An important direction for future work is to select task-specific subset of these generator sets to balance expressivity and efficiency. Another approach is to leverage low-rank or otherwise structured tensor network \cite{PhysRevLett.91.147902,verstraete2004renormalizationalgorithmsquantummanybody,PhysRevLett.101.110501} parameterizations to implement complex equivariant operations more efficiently.


\newpage
\appendix

\section{Proof of Lemma 3.2}\label{sec:proof_3.2}
\begin{proof}[Proof of Lemma 3.2]
    Let $T$ be an $SO(3)$-symmetric tensor whose indices take the 3D representation of $SO(3)$, and $\bm x_i$ be variables each in $\mathbb R^3$. We can define a $SO(3)$-symmetric polynomial
    \begin{equation}
        f(\bm x)=\sum_{i_1,\dots,i_n}T_{i_1,\dots,i_n}(\bm x_1)_{i_1}\cdots(\bm x_n)_{i_n}
    \end{equation}
By Lemma 3.1, we have $f(\bm x)=\sum_i c_i p_i$
where each $p_i$ is product of elements in $\{\bm x_i\cdot \bm x_j,(\bm x_i\times \bm x_j)\cdot \bm x_k\}$ and $c_i$ is the coefficients. Taking derivative of $\bm x_1,\dots,\bm x_n$ on both sides, we can see that $T$ is of the form of finite sum $T=\sum_ic_iT_i$, where $c_i\in \mathbb{R}$ and each $T_i$ is the tensor product of $\delta_{ij}$ and $\epsilon_{ijk}$. 

Considering the parity of the rank, there is odd (even) $\epsilon_{ijk}$ in each $T_i$ if the rank of $T$ is odd (even). Notice that
    \begin{equation}
        \epsilon_{ijk}\epsilon_{lmn}=\delta_{il}(\delta_{jm}\delta_{kn}-\delta_{jn}\delta_{km})-\delta_{im}(\delta_{jl}\delta_{kn}-\delta_{jn}\delta_{kl})+\delta_{in}(\delta_{jl}\delta_{km}-\delta_{jm}\delta_{kl})
    \end{equation}
Then tensor product of odd number of $\epsilon_{ijk}$ reduces to one $\epsilon_{ijk}$. The tensor product of even number of $\epsilon_{ijk}$ reduces to product and sum of the tensor $\delta_{ij}$.
\end{proof}

\section{Proof of Proposition \ref{thm:tensor}}\label{sec:proof_3.3}
\begin{proof}[Proof of Proposition \ref{thm:tensor}]
    It is straightforward to see that each invariant polynomial is a finite sum of homogeneous invariant polynomials. Therefore, we only need to study homogeneous invariant polynomials. Let $p$ be a homogeneous invariant polynomial. Then we can write $p(\bm x_1,\dots, \bm x_n)$ as a tensor network contraction.
    \begin{equation}
        p(\bm x_1,\dots,\bm x_n)=\adjustbox{valign=c}{
        \begin{tikzpicture} 
            \draw (-2.8,0) -- (-2.8,-1.5); 
            \draw (-2,0) -- (-2,-1.5); 
            \draw (-1.2,0) -- (-1.2,-1.5);
            \draw (-0.4,0) -- (-0.4,-1.5);
            \draw (0.4,0) -- (0.4,-1.5);  
            \draw (1.2,0) -- (1.2,-1.5); 
            \draw (2,0) -- (2,-1.5); 
            \draw (2.8,0) -- (2.8,-1.5); 
            \draw (0,0) node[draw,fill=white,rectangle,rounded corners=8pt, minimum width = 165pt] (p) {$T_p$};
            \draw (-2.4,-1.5) node[draw,fill=white,rectangle,rounded corners=4pt, minimum width = 35pt] (x1) {$\bm x_1$};
            \draw (-0.8,-1.5) node[draw,fill=white,rectangle,rounded corners=4pt, minimum width = 35pt] (x2) {$\bm x_1$};
            \draw (0.8,-1.5) node[draw,fill=white,rectangle,rounded corners=4pt, minimum width = 35pt] (x3) {$\bm x_n$};
            \draw (2.4,-1.5) node[draw,fill=white,rectangle,rounded corners=4pt, minimum width = 35pt] (x4) {$\bm x_n$};
            \draw (-2.4,-0.8) node {$\cdots$};
            \draw (-1.6,-0.8) node {$\cdots$};
            \draw (-0.8,-0.8) node {$\cdots$};
            \draw (0,-0.8) node {$\cdots$};
            \draw (0.8,-0.8) node {$\cdots$};
            \draw (1.6,-0.8) node {$\cdots$};
            \draw (2.4,-0.8) node {$\cdots$};
        \end{tikzpicture}
        }
    \end{equation}
    where multiplicity of $\bm x_1,\dots,\bm x_n$  is allowed. Since $p$ is invariant, for $g \in SO(3)$ we have
    \begin{equation}
        p(\bm x_1,\dots,\bm x_n)=p(U(g)^{\otimes r_1} \bm x_1,\dots,U(g)^{\otimes r_n} \bm x_n)
    \end{equation}
    $U$ is the 3D representations of $SO(3)$ and $r_i$ is the rank of $\bm x_i$. That is,
    \begin{equation}
        \adjustbox{valign=c}{
        \begin{tikzpicture}
            \draw (-2.8,0) -- (-2.8,-1.5); 
            \draw (-2,0) -- (-2,-1.5); 
            \draw (-1.2,0) -- (-1.2,-1.5);
            \draw (-0.4,0) -- (-0.4,-1.5);
            \draw (0.4,0) -- (0.4,-1.5);  
            \draw (1.2,0) -- (1.2,-1.5); 
            \draw (2,0) -- (2,-1.5); 
            \draw (2.8,0) -- (2.8,-1.5); 
            \draw (0,0) node[draw,fill=white,rectangle,rounded corners=8pt, minimum width = 165pt] (p) {$T_p$};
            \draw (-2.4,-1.5) node[draw,fill=white,rectangle,rounded corners=4pt, minimum width = 35pt] (x1) {$\bm x_1$};
            \draw (-0.8,-1.5) node[draw,fill=white,rectangle,rounded corners=4pt, minimum width = 35pt] (x2) {$\bm x_1$};
            \draw (0.8,-1.5) node[draw,fill=white,rectangle,rounded corners=4pt, minimum width = 35pt] (x3) {$\bm x_n$};
            \draw (2.4,-1.5) node[draw,fill=white,rectangle,rounded corners=4pt, minimum width = 35pt] (x4) {$\bm x_n$};
            \draw (-2.4,-0.8) node {$\cdots$};
            \draw (-1.6,-0.8) node {$\cdots$};
            \draw (-0.8,-0.8) node {$\cdots$};
            \draw (0,-0.8) node {$\cdots$};
            \draw (0.8,-0.8) node {$\cdots$};
            \draw (1.6,-0.8) node {$\cdots$};
            \draw (2.4,-0.8) node {$\cdots$};
        \end{tikzpicture}
        }=\adjustbox{valign=c}{
        \begin{tikzpicture}
            \draw (-2.8,0) -- (-2.8,-1.5); 
            \draw (-2,0) -- (-2,-1.5); 
            \draw (-1.2,0) -- (-1.2,-1.5);
            \draw (-0.4,0) -- (-0.4,-1.5);
            \draw (0.4,0) -- (0.4,-1.5);  
            \draw (1.2,0) -- (1.2,-1.5); 
            \draw (2,0) -- (2,-1.5); 
            \draw (2.8,0) -- (2.8,-1.5); 
            \draw (0,0) node[draw,fill=white,rectangle,rounded corners=8pt, minimum width = 165pt] (p) {$T_p$};
            \draw (-2.4,-1.5) node[draw,fill=white,rectangle,rounded corners=4pt, minimum width = 35pt] (x1) {$\bm x_1$};
            \draw (-0.8,-1.5) node[draw,fill=white,rectangle,rounded corners=4pt, minimum width = 35pt] (x2) {$\bm x_1$};
            \draw (0.8,-1.5) node[draw,fill=white,rectangle,rounded corners=4pt, minimum width = 35pt] (x3) {$\bm x_n$};
            \draw (2.4,-1.5) node[draw,fill=white,rectangle,rounded corners=4pt, minimum width = 35pt] (x4) {$\bm x_n$};
            \draw (-2.4,-0.8) node {\tiny$\cdots$};
            \draw (-1.6,-0.8) node {\tiny$\cdots$};
            \draw (-0.8,-0.8) node {\tiny$\cdots$};
            \draw (0,-0.8) node {\tiny$\cdots$};
            \draw (0.8,-0.8) node {\tiny$\cdots$};
            \draw (1.6,-0.8) node {\tiny$\cdots$};
            \draw (2.4,-0.8) node {\tiny$\cdots$};
            \draw (-2.8,-0.8) node[draw,circle,fill=white,inner sep = 1pt] {\tiny$U$};
            \draw (-2,-0.8) node[draw,circle,fill=white,inner sep = 1pt] {\tiny$U$};
            \draw (-1.2,-0.8) node[draw,circle,fill=white,inner sep = 1pt] {\tiny$U$};
            \draw (-0.4,-0.8) node[draw,circle,fill=white,inner sep = 1pt] {\tiny$U$};
            \draw (2.8,-0.8) node[draw,circle,fill=white,inner sep = 1pt] {\tiny$U$};
            \draw (2,-0.8) node[draw,circle,fill=white,inner sep = 1pt] {\tiny$U$};
            \draw (1.2,-0.8) node[draw,circle,fill=white,inner sep = 1pt] {\tiny$U$};
            \draw (0.4,-0.8) node[draw,circle,fill=white,inner sep = 1pt] {\tiny$U$};
        \end{tikzpicture}
        }
    \end{equation}
    Taking derivative of $ \bm x_1 \cdots  \bm x_1 \cdots \bm x_n \cdots  \bm x_n$ on both sides, we have
    \begin{equation}\label{eq:derivation}
        \adjustbox{valign=c}{
        \begin{tikzpicture}
            \draw (-2.8,0) -- (-2.8,-2.5); 
            \draw (-2,0) -- (-2,-2.5); 
            \draw (-1.2,0) -- (-1.2,-2.5);
            \draw (-0.4,0) -- (-0.4,-2.5);
            \draw (0.4,0) -- (0.4,-2.5);  
            \draw (1.2,0) -- (1.2,-2.5); 
            \draw (2,0) -- (2,-2.5); 
            \draw (2.8,0) -- (2.8,-2.5); 
            \draw (0,0) node[draw,fill=white,rectangle,rounded corners=8pt, minimum width = 165pt] (p) {$T_p$};
            \draw (-2.4,-0.8) node {$\cdots$};
            \draw (-1.6,-0.8) node {$\cdots$};
            \draw (-0.8,-0.8) node {$\cdots$};
            \draw (0,-0.8) node {$\cdots$};
            \draw (0.8,-0.8) node {$\cdots$};
            \draw (1.6,-0.8) node {$\cdots$};
            \draw (2.4,-0.8) node {$\cdots$};
            \draw (-1.6,-1.5) node[draw,fill=white,rectangle,rounded corners=8pt, minimum width = 80pt,minimum height = 20pt] (P1) {$P_{\rm sym}$};
            \draw (1.6,-1.5) node[draw,fill=white,rectangle,rounded corners=8pt, minimum width = 80pt,minimum height = 20pt] (P2) {$P_{\rm sym}$};
        \end{tikzpicture}
        }=\adjustbox{valign=c}{
        \begin{tikzpicture}
            \draw (-2.8,0) -- (-2.8,-2.5); 
            \draw (-2,0) -- (-2,-2.5); 
            \draw (-1.2,0) -- (-1.2,-2.5);
            \draw (-0.4,0) -- (-0.4,-2.5);
            \draw (0.4,0) -- (0.4,-2.5);  
            \draw (1.2,0) -- (1.2,-2.5); 
            \draw (2,0) -- (2,-2.5); 
            \draw (2.8,0) -- (2.8,-2.5); 
            \draw (0,0) node[draw,fill=white,rectangle,rounded corners=8pt, minimum width = 165pt] (p) {$T_p$};
            \draw (-2.4,-0.8) node {\tiny$\cdots$};
            \draw (-1.6,-0.8) node {\tiny$\cdots$};
            \draw (-0.8,-0.8) node {\tiny$\cdots$};
            \draw (0,-0.8) node {\tiny$\cdots$};
            \draw (0.8,-0.8) node {\tiny$\cdots$};
            \draw (1.6,-0.8) node {\tiny$\cdots$};
            \draw (2.4,-0.8) node {\tiny$\cdots$};
            \draw (-2.8,-0.8) node[draw,circle,fill=white,inner sep = 1pt] {\tiny$U$};
            \draw (-2,-0.8) node[draw,circle,fill=white,inner sep = 1pt] {\tiny$U$};
            \draw (-1.2,-0.8) node[draw,circle,fill=white,inner sep = 1pt] {\tiny$U$};
            \draw (-0.4,-0.8) node[draw,circle,fill=white,inner sep = 1pt] {\tiny$U$};
            \draw (2.8,-0.8) node[draw,circle,fill=white,inner sep = 1pt] {\tiny$U$};
            \draw (2,-0.8) node[draw,circle,fill=white,inner sep = 1pt] {\tiny$U$};
            \draw (1.2,-0.8) node[draw,circle,fill=white,inner sep = 1pt] {\tiny$U$};
            \draw (0.4,-0.8) node[draw,circle,fill=white,inner sep = 1pt] {\tiny$U$};
            \draw (-1.6,-1.7) node[draw,fill=white,rectangle,rounded corners=8pt, minimum width = 80pt,minimum height = 20pt] (P1) {$P_{\rm sym}$};
            \draw (1.6,-1.7) node[draw,fill=white,rectangle,rounded corners=8pt, minimum width = 80pt,minimum height = 20pt] (P2) {$P_{\rm sym}$};
        \end{tikzpicture}
        }
    \end{equation}
    where $P_{\rm sym}$ is the projection to symmetric subspace under permutation within identical $\bm x_i$s.
    \begin{equation}
        {P_{\rm sym}}^{p_{1,1}\dots p_{t_i,r_i}}_{q_{1,1}\dots q_{t_i,r_i}}=\frac 1{t_i!}\sum_{\sigma\in S_{t_i}}\prod_j\prod_k\delta_{q_{\sigma(j),k}}^{p_{j,k}}
    \end{equation}
    where $t_i$ is the multiplicity of $x_i$, and $\sigma$ take value in all permutation of $t_i$ elements.

    It's easy to see that $P_{\rm sym}$ commutes with $U^{\otimes r_i t_i}$
    \begin{equation}\label{eq:commu_P_U}
        \adjustbox{valign=c}{
        \begin{tikzpicture}
            \draw (-2.8,0) -- (-2.8,-2.5); 
            \draw (-2,0) -- (-2,-2.5); 
            \draw (-1.2,0) -- (-1.2,-2.5);
            \draw (-0.4,0) -- (-0.4,-2.5);
            \draw (0.4,0) -- (0.4,-2.5);  
            \draw (1.2,0) -- (1.2,-2.5); 
            \draw (2,0) -- (2,-2.5); 
            \draw (2.8,0) -- (2.8,-2.5); 
            \draw (0,0) node[draw,fill=white,rectangle,rounded corners=8pt, minimum width = 165pt] (p) {$T_p$};
            \draw (-2.4,-0.8) node {\tiny$\cdots$};
            \draw (-1.6,-0.8) node {\tiny$\cdots$};
            \draw (-0.8,-0.8) node {\tiny$\cdots$};
            \draw (0,-0.8) node {\tiny$\cdots$};
            \draw (0.8,-0.8) node {\tiny$\cdots$};
            \draw (1.6,-0.8) node {\tiny$\cdots$};
            \draw (2.4,-0.8) node {\tiny$\cdots$};
            \draw (-2.8,-0.8) node[draw,circle,fill=white,inner sep = 1pt] {\tiny$U$};
            \draw (-2,-0.8) node[draw,circle,fill=white,inner sep = 1pt] {\tiny$U$};
            \draw (-1.2,-0.8) node[draw,circle,fill=white,inner sep = 1pt] {\tiny$U$};
            \draw (-0.4,-0.8) node[draw,circle,fill=white,inner sep = 1pt] {\tiny$U$};
            \draw (2.8,-0.8) node[draw,circle,fill=white,inner sep = 1pt] {\tiny$U$};
            \draw (2,-0.8) node[draw,circle,fill=white,inner sep = 1pt] {\tiny$U$};
            \draw (1.2,-0.8) node[draw,circle,fill=white,inner sep = 1pt] {\tiny$U$};
            \draw (0.4,-0.8) node[draw,circle,fill=white,inner sep = 1pt] {\tiny$U$};
            \draw (-1.6,-1.7) node[draw,fill=white,rectangle,rounded corners=8pt, minimum width = 80pt,minimum height = 20pt] (P1) {$P_{\rm sym}$};
            \draw (1.6,-1.7) node[draw,fill=white,rectangle,rounded corners=8pt, minimum width = 80pt,minimum height = 20pt] (P2) {$P_{\rm sym}$};
        \end{tikzpicture}
        }=\adjustbox{valign=c}{
        \begin{tikzpicture}
            \draw (-2.8,0) -- (-2.8,-2.5); 
            \draw (-2,0) -- (-2,-2.5); 
            \draw (-1.2,0) -- (-1.2,-2.5);
            \draw (-0.4,0) -- (-0.4,-2.5);
            \draw (0.4,0) -- (0.4,-2.5);  
            \draw (1.2,0) -- (1.2,-2.5); 
            \draw (2,0) -- (2,-2.5); 
            \draw (2.8,0) -- (2.8,-2.5); 
            \draw (0,0) node[draw,fill=white,rectangle,rounded corners=8pt, minimum width = 165pt] (p) {$T_p$};
            \draw (-2.4,-2) node {\tiny$\cdots$};
            \draw (-1.6,-2) node {\tiny$\cdots$};
            \draw (-0.8,-2) node {\tiny$\cdots$};
            \draw (0,-2) node {\tiny$\cdots$};
            \draw (0.8,-2) node {\tiny$\cdots$};
            \draw (1.6,-2) node {\tiny$\cdots$};
            \draw (2.4,-2) node {\tiny$\cdots$};
            \draw (-2.8,-2) node[draw,circle,fill=white,inner sep = 1pt] {\tiny$U$};
            \draw (-2,-2) node[draw,circle,fill=white,inner sep = 1pt] {\tiny$U$};
            \draw (-1.2,-2) node[draw,circle,fill=white,inner sep = 1pt] {\tiny$U$};
            \draw (-0.4,-2) node[draw,circle,fill=white,inner sep = 1pt] {\tiny$U$};
            \draw (2.8,-2) node[draw,circle,fill=white,inner sep = 1pt] {\tiny$U$};
            \draw (2,-2) node[draw,circle,fill=white,inner sep = 1pt] {\tiny$U$};
            \draw (1.2,-2) node[draw,circle,fill=white,inner sep = 1pt] {\tiny$U$};
            \draw (0.4,-2) node[draw,circle,fill=white,inner sep = 1pt] {\tiny$U$};
            \draw (-1.6,-1.2) node[draw,fill=white,rectangle,rounded corners=8pt, minimum width = 80pt,minimum height = 20pt] (P1) {$P_{\rm sym}$};
            \draw (1.6,-1.1) node[draw,fill=white,rectangle,rounded corners=8pt, minimum width = 80pt,minimum height = 20pt] (P2) {$P_{\rm sym}$};
        \end{tikzpicture}
        }
    \end{equation}
    Therefore we may define
    \begin{equation}
        \adjustbox{valign=c}{
        \begin{tikzpicture}
            \draw (-2.8,0) -- (-2.8,-1.5); 
            \draw (-2,0) -- (-2,-1.5); 
            \draw (-1.2,0) -- (-1.2,-1.5);
            \draw (-0.4,0) -- (-0.4,-1.5);
            \draw (0.4,0) -- (0.4,-1.5);  
            \draw (1.2,0) -- (1.2,-1.5); 
            \draw (2,0) -- (2,-1.5); 
            \draw (2.8,0) -- (2.8,-1.5); 
            \draw (0,0) node[draw,fill=white,rectangle,rounded corners=8pt, minimum width = 165pt] (p) {$T'_p$};
            \draw (-2.4,-0.8) node {$\cdots$};
            \draw (-1.6,-0.8) node {$\cdots$};
            \draw (-0.8,-0.8) node {$\cdots$};
            \draw (0,-0.8) node {$\cdots$};
            \draw (0.8,-0.8) node {$\cdots$};
            \draw (1.6,-0.8) node {$\cdots$};
            \draw (2.4,-0.8) node {$\cdots$};
        \end{tikzpicture}
        }=
        \adjustbox{valign=c}{
        \begin{tikzpicture}
            \draw (-2.8,0) -- (-2.8,-2.5); 
            \draw (-2,0) -- (-2,-2.5); 
            \draw (-1.2,0) -- (-1.2,-2.5);
            \draw (-0.4,0) -- (-0.4,-2.5);
            \draw (0.4,0) -- (0.4,-2.5);  
            \draw (1.2,0) -- (1.2,-2.5); 
            \draw (2,0) -- (2,-2.5); 
            \draw (2.8,0) -- (2.8,-2.5); 
            \draw (0,0) node[draw,fill=white,rectangle,rounded corners=8pt, minimum width = 165pt] (p) {$T_p$};
            \draw (-2.4,-0.8) node {$\cdots$};
            \draw (-1.6,-0.8) node {$\cdots$};
            \draw (-0.8,-0.8) node {$\cdots$};
            \draw (0,-0.8) node {$\cdots$};
            \draw (0.8,-0.8) node {$\cdots$};
            \draw (1.6,-0.8) node {$\cdots$};
            \draw (2.4,-0.8) node {$\cdots$};
            \draw (-1.6,-1.5) node[draw,fill=white,rectangle,rounded corners=8pt, minimum width = 80pt,minimum height = 20pt] (P1) {$P_{\rm sym}$};
            \draw (1.6,-1.5) node[draw,fill=white,rectangle,rounded corners=8pt, minimum width = 80pt,minimum height = 20pt] (P2) {$P_{\rm sym}$};
        \end{tikzpicture}
        }
    \end{equation}
    Combining Eq.(\ref{eq:derivation}) and Eq.(\ref{eq:commu_P_U}), it is straightforward to see that $T'_p$ is an $SO(3)$-symmetric tensor. By Lemma 3.2, $T'_p$ is a linear combination of the tensor product of delta tensors and at most one Levi-Civita tensor. For polynomial $p(\bm x_1,\dots,\bm x_n)$, we can always use the partially permutation symmetrization of the tensor $T_p$ as its coefficients, that is,
    \begin{equation}
        p(\bm x_1,\dots,\bm x_n)=\adjustbox{valign=c}{
        \begin{tikzpicture}
            \draw (-2.8,0) -- (-2.8,-1.5); 
            \draw (-2,0) -- (-2,-1.5); 
            \draw (-1.2,0) -- (-1.2,-1.5);
            \draw (-0.4,0) -- (-0.4,-1.5);
            \draw (0.4,0) -- (0.4,-1.5);  
            \draw (1.2,0) -- (1.2,-1.5); 
            \draw (2,0) -- (2,-1.5); 
            \draw (2.8,0) -- (2.8,-1.5); 
            \draw (0,0) node[draw,fill=white,rectangle,rounded corners=8pt, minimum width = 165pt] (p) {$T'_p$};
            \draw (-2.4,-1.5) node[draw,fill=white,rectangle,rounded corners=4pt, minimum width = 35pt] (x1) {$\bm x_1$};
            \draw (-0.8,-1.5) node[draw,fill=white,rectangle,rounded corners=4pt, minimum width = 35pt] (x2) {$\bm x_1$};
            \draw (0.8,-1.5) node[draw,fill=white,rectangle,rounded corners=4pt, minimum width = 35pt] (x3) {$\bm x_n$};
            \draw (2.4,-1.5) node[draw,fill=white,rectangle,rounded corners=4pt, minimum width = 35pt] (x4) {$\bm x_n$};
            \draw (-2.4,-0.8) node {$\cdots$};
            \draw (-1.6,-0.8) node {$\cdots$};
            \draw (-0.8,-0.8) node {$\cdots$};
            \draw (0,-0.8) node {$\cdots$};
            \draw (0.8,-0.8) node {$\cdots$};
            \draw (1.6,-0.8) node {$\cdots$};
            \draw (2.4,-0.8) node {$\cdots$};
        \end{tikzpicture}
        }
    \end{equation}
    
     Therefore, $p(\bm x_1,\dots,\bm x_n)$ is a linear combination of the contraction  of  tensor networks formed by $\bm x_1,\dots,\bm x_n$(multiplicity is allowed) together with at most one Levi-Civita tensor $\epsilon_{ijk}$. Since contraction of disconnected tensor network is product of the contraction of each component, $\mathbb R[V]^{SO(3)}$ is generated by the contraction of connected tensor network formed by $\bm x_1,\dots,\bm x_n$ (multiplicity is allowed) together with at most one Levi-Civita tensor $\epsilon_{ijk}$.
\end{proof}

\section{Decomposition of tensor product representation $(1)^{\otimes l}$ }\label{sec:product_dec}
We have
\begin{equation}
    (1)^{\otimes l}=(l)^{d_{l,l}}\oplus(l-1)^{d_{l,l-1}}\oplus\cdots\oplus(j)^{d_{l,j}}\oplus\cdots\oplus(0)^{d_{l,0}}
\end{equation}
By counting the dimension of $L_z$ eigen-spaces \footnote{$L_x, L_y, L_z$ is the angular momentum operator of spin-1 particle in x, y, z direction, respectively\citep{zee2016group}.}, we have
\begin{equation}
    d_{l,j}=\sum_{i=j}^{\left\lfloor \frac {j+l}2\right\rfloor}\frac{l!}{i!(j+l-2i)!(i-j)!}-\sum_{i=j+1}^{\left\lfloor \frac {j+l+1}2\right\rfloor}\frac{l!}{i!(j+l+1-2i)!(i-j-1)!}.
\end{equation}
Particularly, $d_{l,l}=1$, $d_{l,l-1}=l-1$ and $d_{l,l-2}=\frac{l(l-1)}2$. When $l$ is small, the exact decomposition is listed below
\begin{align}
    (1)^{\otimes 2} &= (2)\oplus(1)\oplus(0)\\
    (1)^{\otimes 3}  &= (3)\oplus(2)^2\oplus(1)^3\oplus(0)\\
    (1)^{\otimes 4}  &= (4)\oplus(3)^3\oplus(2)^6\oplus(1)^6\oplus(0)^3\\
    (1)^{\otimes 5}  &= (5)\oplus (4)^4\oplus(3)^{10}\oplus(2)^{15}\oplus(1)^{15}\oplus(0)^6
\end{align}

\section{The exact form of tensor $P_l$}\label{sec:projector_form}

For $l$ representation, we have basis vector $|m\rangle$ ($-l\leq m \leq l$). The $l$ sub-representation in $(1)^{\otimes l}$ satisfies
\begin{equation}
    \sqrt{\frac{(2l)!(l-m)!}{(l+m)!}}|m\rangle= (L_-)^{l-m} |l\rangle =\left(\sum_i L_{-,i}\right)^{l-m}|1,\dots,1\rangle
\end{equation}
where $L_\pm = L_x\pm iL_y$ is the ladder operator.\footnote{$L_-|m\rangle = \sqrt{(l+m)(l-m+1)} |m-1\rangle$ and $L_-|-l\rangle=0$, where $|m\rangle$ is the standard basis of $l$ representation. Especially, for 1 representation, $L_-|1\rangle = \sqrt 2 |0\rangle$, $L_-|0\rangle = \sqrt 2 |-1\rangle$ and $L_-|-1\rangle = 0$ \citep{zee2016group}.}.

Then
\begin{align}
    |m\rangle&=\sqrt{\frac{(l+m)!}{(2l)!(l-m)!}} (\sqrt 2)^{l-m} \sum_{(s_i),\sum s_i = m} \frac{(l-m)!}{2^{d(s)}}|s_1,\dots,s_l\rangle \\
    &=\sqrt{\frac{(l-m)!(l+m)!2^{l-m}}{(2l)!}} \sum_{(s_i),\sum s_i = m} \frac1{2^{d(s)}}|s_1,\dots,s_l\rangle
\end{align}
where $s_i=0,\pm 1$ and $d(s)$ is the number of $-1$ in $(s)=(s_1,\dots,s_l)$. 

Therefore
\begin{equation}\label{eq:p_l}
    [P_l]_{(s)}^m = \frac1{2^{d(s)}}\sqrt{\frac{(l-m)!(l+m)!2^{l-m}}{(2l)!}} \delta_{\sum s_i,m}
\end{equation}

Equation \ref{eq:p_l} gives the components of $P_l$ in the spherical (spin-1) basis
$\{|+1\rangle,|0\rangle,|-1\rangle\}$.
To obtain the Cartesian components, we apply the basis change on each spin-1 leg.
Specifically, define $U\in\mathbb{C}^{3\times 3}$ by
\begin{equation}
   |s\rangle=\sum_{a\in\{x,y,z\}}U_{sa}\,|a\rangle,\qquad s\in\{+1,0,-1\}, 
\end{equation}
which is equivalent to
\begin{equation}
    |+1\rangle=-\frac{1}{\sqrt2}(|x\rangle+i|y\rangle),\qquad
|0\rangle=|z\rangle, \qquad
|-1\rangle=\frac{1}{\sqrt2}(|x\rangle-i|y\rangle),
\end{equation}
and hence
\begin{equation}
    U=\begin{pmatrix}
-\tfrac{1}{\sqrt2} & -\tfrac{i}{\sqrt2} & 0\\[3pt]
0 & 0 & 1\\[3pt]
\tfrac{1}{\sqrt2} & -\tfrac{i}{\sqrt2} & 0
\end{pmatrix},
\end{equation}
Then the Cartesian components of $P_l$ are given by
\begin{equation}
    [P_l]^m_{a_1\cdots a_l}
=[P_l]^m_{s_1\cdots s_l}\,
[U]_{s_1 a_1}\cdots [U]_{s_l a_l},
\qquad a_i\in\{x,y,z\}.
\end{equation}
One can also further choose a real basis for the type-$l$ output space by applying an
additional change-of-basis matrix $R^{(l)}$ on the output index, i.e., $[\widetilde P_l]^{\mu}_{a_1\cdots a_l} = [R^{(l)}]_{\mu m}\,[P_l]^m_{a_1\cdots a_l}$, where $\mu$ indexes the chosen real type-$l$ basis.

$P_l$ have some nice properties as show in lemma \ref{P_l_proper}. We give the proof as follows:
\begin{proof}[Proof of lemma \ref{P_l_proper} ]
\begin{enumerate}
     \item[1,2] By the definition of reduction of representation.
     \item[3] Contraction with $\delta_{ij}$ induces an  $SO(3)$ symmetric map $(l)\rightarrow (1)^{\otimes(l-2)}$ which is zero.
     \item[4] Contraction with $\epsilon_{ijk}$ induces an  $SO(3)$ symmetric map $(l)\rightarrow (1)^{\otimes(l-1)}$ which is zero.
     \item[5] From the equation $\epsilon_{ijk}\epsilon_{kpq}=\delta_{ip}\delta_{jq}-\delta_{iq}\delta_{jp}$, we have 
     \begin{align}
         [P_l]^{m}_{ \dots, a_j, \dots a_i,\dots} = [P_l]^{m}_{ \dots, a_i, \dots a_j,\dots} -\nonumber
          [P_l]^{m}_{ \dots, a'_i, \dots a'_j,\dots} \epsilon_{a'_ia'_jk}\epsilon_{ka_ia_j}
     \end{align}
     Then we may apply 4.
\end{enumerate}
\end{proof}

\section{Proof of Proposition \ref{Spher_gene}} \label{sec:proof_3.4}
\begin{proof}[Proof of Proposition \ref{Spher_gene}]
    Let $p$ be a homogeneous invariant polynomial. Similar to the proof of Proposition \ref{thm:tensor}, we have
    \begin{equation}
        p(\bm x_1,\dots,\bm x_n)=\adjustbox{valign=c}{
        \begin{tikzpicture}
            \draw (0,0) node[draw,rectangle,rounded corners=8pt, minimum width = 50pt] (p) {$T'_p$};
            \draw (-2.4,-1.5) node[draw,circle] (x1) {$\bm x_1$};
            \draw (-0.8,-1.5) node[draw,circle] (x2) {$\bm x_1$};
            \draw (0.8,-1.5) node[draw,circle] (x3) {$\bm x_n$};
            \draw (2.4,-1.5) node[draw,circle] (x4) {$\bm x_n$};
            \draw (-1.6,-1.5) node {$\cdots$};
            \draw (0,-1.5) node {$\cdots$};
            \draw (1.6,-1.5) node {$\cdots$};
            \draw (p) to[out=-170,in=90,looseness=0.5] (x1);
            \draw (p) to[out=-140,in=90,looseness=0.5] (x2);
            \draw (p) to[out=-40,in=90,looseness=0.5] (x3);
            \draw (p) to[out=-10,in=90,looseness=0.5] (x4);
        \end{tikzpicture}
        }
    \end{equation}
    where $T'_p$ is an $SO(3)$-symmetric tensor. Inserting identities, we have
    \begin{equation}
        p(\bm x_1,\dots,\bm x_n)=\adjustbox{valign=c}{
        \begin{tikzpicture}
            \draw (0,0) node[draw,rectangle,rounded corners=8pt, minimum width = 50pt] (p) {$T'_p$};
            \draw (-2.4,-1.5) node[draw,circle] (P1) {$P_{l_1}$};
            \draw (-0.8,-1.5) node[draw,circle] (P2) {$P_{l_1}$};
            \draw (0.8,-1.5) node[draw,circle] (P3) {$P_{l_n}$};
            \draw (2.4,-1.5) node[draw,circle] (P4) {$P_{l_n}$};
            \draw (-2.4,-3.5) node[draw,circle] (x1) {$P_{l_1}(\bm x_1)$};
            \draw (-0.8,-3.5) node[draw,circle] (x2) {$P_{l_1}(\bm x_1)$};
            \draw (0.8,-3.5) node[draw,circle] (x3) {$P_{l_n}(\bm x_n)$};
            \draw (2.4,-3.5) node[draw,circle] (x4) {$P_{l_n}(\bm x_n)$};
            \draw (-1.6,-1.5) node {$\cdots$};
            \draw (0,-1.5) node {$\cdots$};
            \draw (1.6,-1.5) node {$\cdots$};
            \draw (p) to[out=-170,in=90,looseness=0.5] (P1);
            \draw (p) to[out=-140,in=90,looseness=0.5] (P2);
            \draw (p) to[out=-40,in=90,looseness=0.5] (P3);
            \draw (p) to[out=-10,in=90,looseness=0.5] (P4);
            \draw (P1) to[out=-120,in=120] (x1); 
            \draw (P1) to[out=-60,in=60] (x1); 
            \draw (P2) to[out=-120,in=120] (x2); 
            \draw (P2) to[out=-60,in=60] (x2); 
            \draw (P3) to[out=-120,in=120] (x3); 
            \draw (P3) to[out=-60,in=60] (x3); 
            \draw (P4) to[out=-120,in=120] (x4); 
            \draw (P4) to[out=-60,in=60] (x4); 
            \draw (-2.4,-2.5) node {$\cdots$};
            \draw (-0.8,-2.5) node {$\cdots$};
            \draw (0.8,-2.5) node {$\cdots$};
            \draw (2.4,-2.5) node {$\cdots$};
        \end{tikzpicture}
        }
    \end{equation}
    where the tensor network
    \begin{equation}
        \adjustbox{valign=c}{\begin{tikzpicture}
            \draw (0,0) node[draw,rectangle,rounded corners=8pt, minimum width = 50pt] (p) {$T'_p$};
            \draw (-2.4,-1.5) node[draw,circle] (P1) {$P_{l_1}$};
            \draw (-0.8,-1.5) node[draw,circle] (P2) {$P_{l_1}$};
            \draw (0.8,-1.5) node[draw,circle] (P3) {$P_{l_n}$};
            \draw (2.4,-1.5) node[draw,circle] (P4) {$P_{l_n}$};
            \draw (-1.6,-1.5) node {$\cdots$};
            \draw (0,-1.5) node {$\cdots$};
            \draw (1.6,-1.5) node {$\cdots$};
            \draw (p) to[out=-170,in=90,looseness=0.5] (P1);
            \draw (p) to[out=-140,in=90,looseness=0.5] (P2);
            \draw (p) to[out=-40,in=90,looseness=0.5] (P3);
            \draw (p) to[out=-10,in=90,looseness=0.5] (P4);
            \draw (P1) to[out=-120,in=90] (-2.8,-2.5); 
            \draw (P1) to[out=-60,in=90] (-2,-2.5); 
            \draw (P2) to[out=-120,in=90] (-1.2,-2.5); 
            \draw (P2) to[out=-60,in=90] (-0.4,-2.5); 
            \draw (P3) to[out=-120,in=90] (0.4,-2.5); 
            \draw (P3) to[out=-60,in=90] (1.2,-2.5); 
            \draw (P4) to[out=-120,in=90] (2,-2.5); 
            \draw (P4) to[out=-60,in=90] (2.8,-2.5); 
            \draw (-2.4,-2.5) node {$\cdots$};
            \draw (-0.8,-2.5) node {$\cdots$};
            \draw (0.8,-2.5) node {$\cdots$};
            \draw (2.4,-2.5) node {$\cdots$};
        \end{tikzpicture}}
    \end{equation}
    is $SO(3)$-symmetric and the free legs are of 3D representation. By Lemma 3.2, the contraction of this tensor network is linear combination of tensor product of delta tensors and at most one Levi-Civita tensor. Therefore $p(\bm x_1,\dots,\bm x_n)$ is linear combination of contraction  of  tensor network formed by $P_{l_1}(\bm x_1),\dots,P_{l_n}(\bm x_n)$ (multiplicity is allowed) together with at most one Levi-Civita tensor $\epsilon_{ijk}$. Since contraction of disconnected tensor network is product of the contraction of each component, $\mathbb R[V]^{SO(3)}$ is generated by the contraction of connected tensor network formed by $P_{l_1}(\bm x_1),\dots,P_{l_n}(\bm x_n)$ (multiplicity is allowed) together with at most one Levi-Civita tensor $\epsilon_{ijk}$.
\end{proof}

\section{Proof of Proposition \ref{cont_invar_fun}}\label{sec:proof_cont_invar_fun}
\begin{proof}[Proof of Proposition \ref{cont_invar_fun}]
Let a compact set $K\subset \bigoplus_j V_j$ and $\eta>0$. Denote input ${\bm x}=(\bm x_1,\dots, \bm x_n)$ and $R\in SO(3)$. By the Stone–Weierstrass theorem, there exists a polynomial $p$ such that
\begin{equation}
    \sup_{\bm x\in K}|f(\bm x)-p(\bm x)|<\eta.
\end{equation}
Define the group-averaged polynomial
\begin{equation}
    \bar p(\bm x):=\int_{SO(3)} p(R\cdot \bm x)dR,
\end{equation}
where $dR$ is the normalized Haar measure. Then $\bar p$ is still a polynomial and is exactly $SO(3)$ invariant. Moreover, since $f$ is invariant, we have $f(R\cdot \bm x)=f(\bm x)$ for all $R$, hence
\begin{equation}
    |f(\bm x)-\bar p(\bm x)|
    =\left|\int_{SO(3)} \big(f(R\cdot \bm x)-p(R\cdot \bm x)\big)dR\right|
    \le \int_{SO(3)} |f(R\cdot \bm x)-p(R\cdot \bm x)|dR
    \le \sup_{y\in K'}|f(y)-p(y)|,
\end{equation}
where $K':=SO(3)\cdot K$ is compact. Therefore, choosing $\bar p$ to approximate $f$ uniformly on $K'$ yields
\begin{equation}
    \sup_{\bm x\in K'}|f(\bm x)-\bar p(\bm x)|<\eta.
\end{equation}
Finally, $\bar p\in \mathbb{R}[V]^{SO(3)}$. Let $g_1,\ldots,g_k$ be the tensor network generators of $\mathbb{R}[V]^{SO(3)}$. There exists an ordinary polynomial $q$ such that
\begin{equation}
    \bar p(\bm x)=q\big(g_1(\bm x),\ldots,g_k(\bm x)\big).
\end{equation}
Setting $\tilde f(\bm x):=q(g_1(\bm x),\ldots,g_k(\bm x))$ completes the proof.

\end{proof}

\section{Obtaining the equivariant functions from the invariant functions }\label{sec:equ_from_in}
We can give the proof of Lemma \ref{th_equ_f_in} as following:
\begin{proof}[Proof of Lemma \ref{th_equ_f_in}]
 Since $f$ is invariant, for each $R\in G$,

\begin{align}
 T_{\rm up}(f)^{i}(\bm x_1, \cdots,\bm x_n)_\alpha =&\left. \frac{\partial f(\bm x_1, \cdots,\bm x_n, \bm y_1,\cdots,\bm y_m)}{\partial (\bm y_i)_\alpha} \right|_{\bm y_1=\cdots=\bm y_m=0}\\
  =& \left. \frac{ \partial f(R \cdot \bm x_1, \cdots, R \cdot \bm x_n,  R \cdot \bm y_1,\cdots,R \cdot \bm y_m)}{\partial (\bm y_i)_\alpha}\right|_{\bm y_1=\cdots=\bm y_m=0} \\
  =& \sum_\beta \left. \frac{ \partial f(R \cdot \bm x_1, \cdots, R \cdot \bm x_n, \bm y'_1,\cdots,\bm y'_m)}{\partial (\bm y'_i)_\beta}\right|_{\bm y'_1=\cdots=\bm y'_m=0}\frac{\partial( \bm y'_i)_\beta}{\partial(\bm y_i)_\alpha} \\
  =& \sum_\beta T_{\rm up}(f)^{i}(R \cdot \bm x_1, \cdots, R \cdot \bm x_n)_\beta \rho_i(R)_{\beta\alpha}
\end{align}
where $\bm y'_i=R \cdot \bm y_i$.

Therefore
\begin{align}
     T_{\rm up}(f)^{i}(R \cdot \bm x_1, \cdots, R \cdot \bm x_n)_\alpha =&\sum_\beta  T_{\rm up}(f)^{i}(\bm x_1, \cdots, \bm x_n)_\beta \rho_i(R^{-1})_{\beta\alpha}\\
     =&\sum_\beta \bar\rho_i(R)_{\alpha\beta} T_{\rm up}(f)^{i}(\bm x_1, \cdots, \bm x_n)_\beta \\
     =& (R \cdot T_{\rm up}(f)^{i}(\bm x_1, \cdots, \bm x_n))_\alpha
\end{align}
$T_{\rm up}(f)$ is equivariant.

To prove that any equivariant function can be obtained in this way, we consider the following construction. Given an equivariant function $f: \bigoplus_j V_j  \rightarrow\bigoplus_i U_i$ with input $\bm x_1,\dots,\bm x_n$ in space $V_1,\dots,V_n$ and output $\bm y_1,\dots,\bm y_m$ in space $U_1,\dots,U_m$, we can construct an invariant function $T_{\rm down}(f):\bigoplus_j V_j\oplus\bigoplus_i \bar U_i\rightarrow F$, where $G$ acts on $\bar U_i$ by the dual representation of $U_i$, by defining 
\begin{equation}
    T_{\rm down}(f)(\bm x_1, \cdots, \bm x_n,\bm y_1,\dots,\bm y_m) = \sum_i \left \langle f^i(\bm x_1, \cdots, \bm x_n), \bm y_i  \right \rangle \label{eqn:equiv2inv}
\end{equation}
where we have choose a natural set of basis for $U_i$ and the corresponding dual basis for $\bar U_i$,  $\left \langle \cdot  , \cdot \right \rangle$ denotes the natural function $U_i\times \bar U_i\rightarrow \mathbb R$.

By simple deduction, one can show that $T_{\rm up}\circ T_{\rm down} (f)=f$. Therefore, any equivariant function $f$ can be constructed by $T_{\rm up}(h)$ where $h= T_{\rm down} (f)$.
\end{proof}

Notice that for $S U(2)$ and $S O(3)$, each representation is similar to its dual. 
The transformation $ T_{\rm up}$ and $ T_{\rm down}$ can be pictorially expressed by Fig.\ref{fig:bend_up_down}.  
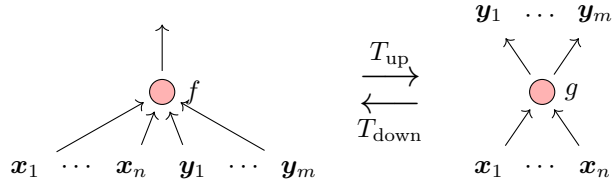
\begin{figure}[htb!]
    \centering
    \begin{tikzpicture}
        \begin{scope}
            \draw (0,0) node[draw, circle,fill = red!30,minimum width = 5pt] (f) {} ;
            \path (f) node[right = 5pt]  {$f$};
            \draw(-1.8,-0.8) node[anchor=north] (x1) {$\bm x_1$};
            \draw(-0.4,-0.8) node[anchor=north] (x2) {$\bm x_n$};
            \draw(1.8,-0.8) node[anchor=north] (y1) {$\bm y_m$};
            \draw(0.4,-0.8) node[anchor=north] (y2) {$\bm y_1$};
            \draw(-1.1,-0.8) node[anchor=north] {$\cdots$};
            \draw(1.1,-0.8) node[anchor=north] {$\cdots$};
            \draw[->,shorten >=3pt,shorten <=3pt] (x1) -- (f);
            \draw[->,shorten >=3pt,shorten <=3pt] (x2) -- (f);
            \draw[->,shorten >=3pt,shorten <=3pt] (y1) -- (f);
            \draw[->,shorten >=3pt,shorten <=3pt] (y2) -- (f);
            \draw[->,shorten >=3pt,shorten <=3pt] (f) -- (0,1);
        \end{scope}
        \draw (3,0.4)node {\Large{$\overset{T_{\rm up}}\longrightarrow$}};
        \draw (3,-0.4)node {\Large{$\underset{T_{\rm down}}\longleftarrow$}};
        \begin{scope}[xshift=5cm]
            \draw (0,0) node[draw, circle,fill = red!30,minimum width = 5pt] (g) {} ;
            \path (g) node[right = 5pt]  {$g$};
            \draw(-0.7,-0.8) node[anchor=north] (x3) {$\bm x_1$};
            \draw(0.7,-0.8) node[anchor=north] (x4) {$\bm x_n$};
            \draw(0,-0.8) node[anchor=north] {$\cdots$};
            \draw(-0.7,0.8) node[anchor=south] (y3) {$\bm y_1$};
            \draw(0.7,0.8) node[anchor=south] (y4) {$\bm y_m$};
            \draw(0,0.8) node[anchor=south] {$\cdots$};
            \draw[->,shorten >=3pt,shorten <=3pt] (x3) -- (g);
            \draw[->,shorten >=3pt,shorten <=3pt] (x4) -- (g);
            \draw[->,shorten >=3pt,shorten <=3pt] (g) -- (y3);
            \draw[->,shorten >=3pt,shorten <=3pt] (g) -- (y4);
        \end{scope}
    \end{tikzpicture}
    \caption{$T_{\rm up}$ transforms an invariant function into an equivariant function. $T_{\rm down}$ transforms an equivariant function into an invariant function.}
    \label{fig:bend_up_down}
\end{figure}

\section{Proof of Proposition \ref{equ_tns}}\label{proof_equ_tns}
\begin{proof}[Proof of Proposition \ref{equ_tns}]
In the first part of the proof, we show that any continuous $SO(3)$ equivariant function can be approximated by $SO(3)$ equivariant polynomials. We adopt the group averaging method for equivariant maps, following the approach described in \cite{Yarotsky2022}.

Let $V:=\bigoplus_j V_j$ and $U:=\bigoplus_i U_i$. Let $K\subset V$ be a compact set and $\varepsilon>0$. Denote the input $\bm x=(\bm x_1,\ldots,\bm x_n)$ and the actions on $V$ and $U$ by
$R^V$ and $R^U$, respectively. The $SO(3)$ equivariance of $h:V\to U$ means that for all
$R\in SO(3)$,
\begin{equation}
    (R^U)^{-1}h(R^V  \bm x)=h(\bm x).
\end{equation}

By the Stone--Weierstrass theorem applied componentwise, there exists a polynomial map $p:V\to U$ such that
\begin{equation}
    \sup_{\bm x\in K'}\|h(\bm x)-p(\bm x)\|<\delta,
\end{equation}
where $K':=SO(3)\cdot K$ is compact. Define the group-averaged polynomial
\begin{equation}
    \bar p(\bm x):=\int_{SO(3)} (R^U)^{-1}p(R^V \bm x)\,dR,
\end{equation}
Then $\bar p$ is still a polynomial and is
exactly $SO(3)$ equivariant. Moreover, since $h$ is equivariant, we have
$(R^U)^{-1}h(R^V \bm x)=h(\bm x)$ for all $R$, hence for any $\bm x\in K'$,
\begin{equation}
\begin{aligned}
    \|\bar p(\bm x)-h(\bm x)\|
    &=\left\|\int_{SO(3)} (R^U)^{-1}\big(p(R^V \bm x)-h(R^V \bm x)\big)\,dR\right\| \\
    &\le \int_{SO(3)} \|(R^U)^{-1}\|\,\|p(R^V \bm x)-h(R^V \bm x)\|\,dR \\
    &\le C_U \sup_{z\in K'}\|p(\bm z)-h(\bm z)\|,
\end{aligned}
\end{equation}
where $C_U:=\sup_{R\in SO(3)}\|(R^U)^{-1}\|<\infty$.
Therefore,
\begin{equation}
    \sup_{x\in K'}\|\bar p(\bm x)-h(\bm x)\|
    \le C_U \sup_{z\in K'}\|p(\bm z)-h(\bm z)\|
    < C_U\delta.
\end{equation}
Choosing $\delta=\varepsilon/C_U$ yields
\begin{equation} \label{eq:app_gene_fun}
    \sup_{x\in K'}\|\bar p(\bm x)-h(\bm x)\|<\varepsilon,
\end{equation}

In the second part of the proof, we show that any $SO(3)$ equivariant polynomials can be written as
\begin{equation}\label{Eq:equ_poly}
    \tilde{h}^i(\bm x) = \sum_j ^{N_i}q^i_{j}(g_1,\dots,g_k) t^i_j(\bm x)
\end{equation}
According to Lemma \ref{th_equ_f_in}, the equivariant polynomials can be obtained by differentiation of invariant polynomials. We can classify the tensor network generators into three types
\begin{enumerate}
    \item $G_1$: tensor network generators that only contain $\bm x_i$
    \item $G_2$: tensor network generators that contain exactly one $\bm y_i$
    \item $G_3$: tensor network generators that contain more than one $\bm y_i$
\end{enumerate}
Then we have
\begin{equation}
    \tilde{h}^{i}(\bm x) =\left. \frac{\partial f(\{ g_u\}, \{\bar { t}^v_k\}, \{ s_w\})}{\partial \bm y_i} \right|_{\bar {t}^v_j= s_w=0} 
\end{equation}
where $ g_u\in G_1$, $\bar {t}^v_k\in G_2$ ($v$ means the generator contains exactly one $\bm y_v$), $s_w\in G_3$.

It's easy to see that 
\begin{align}
    h^{i}(\bm x) &=\sum_j\left. \frac{\partial f(\{g_u\}, \{\bar { t}^v_k\}, \{ s_w\})}{\partial \bar{ t}^i_j}\right|_{\bar { t}^v_j= s_w=0} \frac{\partial\bar { t}^i_j}{\partial \bm y_i} \\
    &=\sum_jq^i_j(\{ g_u\})  t^i_j
\end{align}
where $q^i_j(\{ g_u\})$ is a function that represents $\left. \frac{\partial f(\{ g_u\}, \{\bar { t}^v_k\}, \{ s_w\})}{\partial \bar{ t}^i_j}\right|_{\bar { t}^v_j= s_w=0}$, and $ t^i_j$ is $\bar { t}^i_j$ with $\bm y_i$ missing.

We note $\bar p$ is the $SO(3)$ equivariant polynomials. Combining Eq.(\ref{eq:app_gene_fun}) and Eq.(\ref{Eq:equ_poly}), we completes the proof.
\end{proof}

\section{Proof of the Proposition \ref{spher_eq}}\label{Proof_spher_eq}
\begin{proof}[proof of the Proposition \ref{spher_eq}]
Generally, $t_i$ can be obtained by contracting tensor networks built from $\{\bm x_1,\ldots,\bm x_n\}$, a single $P_l$, copies of $\delta_{ij}$, and at most one $\epsilon_{ijk}$. We now show that the properties of $P_l$ results in the restricted topologies in \ref{eq:connect_spher}.

By Lemma \ref{P_l_proper} (3) (traceless), contracting any two input legs of $P_l$ with $\delta_{ij}$
yields 0. Hence no contraction can connect two legs of $P_l$ directly and every leg of $P_l$ must be contracted with indices belonging to other tensors, which are ultimately the input vectors and possibly
$\epsilon_{ijk}$. Moreover, by Lemma \ref{P_l_proper} (4), contracting two indices of $\epsilon_{ijk}$ with two legs of $P_l$ also yields 0. Consequently, if an $\epsilon_{ijk}$ is present in the network, it can share at most one index with $P_l$; the other two indices of $\epsilon_{ijk}$ must contract
with indices coming from the input vectors. Combining the above constraints, any non-vanishing network must be of one of the two forms in \ref{eq:connect_spher}.
\end{proof}

\section{Represent spherical tensor equivariant function by tensor network}\label{sec:sph_equiv_example}

The TP operations can be expressed by
\begin{equation}
    \bm c_t=\sum_{rs} M_{rst}\bm a_r \bm b_s
\end{equation}
where the type of $\bm a$, $\bm b$ and $\bm c$ are $l_a$, $l_b$ and $l_c$ respectively,  $1\leq r\leq 2l_a+1,1\leq s\leq 2l_b+1,1\leq t\leq 2l_c+1$ and $M$ is a rank-3 symmetric tensor, whose indices are of representation $l_a,l_b$ and $l_c$.

$M_{rst}$ is proportional $C_{rst}$, which is the unique symmetric projection tensor $(l_a)\otimes (l_b)\rightarrow(l_c)$ (also called CG coefficients). In other words, we have
\begin{equation}
    \bm c_t=\sum_{rs} AC_{rst}\bm a_r \bm b_s \label{eqn:sphe2}
\end{equation}

Following the construction method of Proposition \ref{Spher_gene}, we have

\begin{equation}
    \bm c_t =A'
    \adjustbox{valign = c, scale=0.8}{\begin{tikzpicture}
        \draw (90:1) node[draw,circle] (pc) {$P_{l_c}$};
        \draw (-30:1) node[draw,circle] (pb) {$P_{l_b}$};
        \draw (210:1) node[draw,circle] (pa) {$P_{l_a}$};
        \coordinate (c) at (90:2.3) ;
        \draw (-30:2.3) node[draw,circle] (b) {$\bm b$};
        \draw (210:2.3) node[draw,circle] (a) {$\bm a$};
        \draw (pa) edge["\footnotesize$l_a$"'] (a);
        \draw (pb) edge["\footnotesize$l_b$"] (b);
        \draw (pc) edge["\footnotesize$l_c$"] (c);
        \draw (pa) edge[ultra thick] (pb);
        \draw (pb) edge[ultra thick] (pc);
        \draw (pc) edge[ultra thick] (pa);
        \node at (2.25,0) {\textit{or}};
    \end{tikzpicture}}
    \qquad
    \adjustbox{valign = c, scale=0.8}{\begin{tikzpicture}
        \node[circle, fill, inner sep=1.5pt] (fill) at (0,0) {}; 
        \draw (90:1) node[draw,circle] (pc) {$P_{l_c}$};
        \draw (-30:1) node[draw,circle] (pb) {$P_{l_b}$};
        \draw (210:1) node[draw,circle] (pa) {$P_{l_a}$};
        \coordinate (c) at (90:2.3) ;
        \draw (-30:2.3) node[draw,circle] (b) {$\bm b$};
        \draw (210:2.3) node[draw,circle] (a) {$\bm a$};        
        \draw (pa) edge["\footnotesize$l_a$"'] (a);
        \draw (pb) edge["\footnotesize$l_b$"] (b);
        \draw (pc) edge["\footnotesize$l_c$"] (c);     
        \draw (pa) edge[ultra thick] (pb);
        \draw (pb) edge[ultra thick] (pc);
        \draw (pc) edge[ultra thick] (pa);
        \draw (fill) edge (pa);
        \draw (fill) edge (pb);
        \draw (fill) edge (pc);
    \end{tikzpicture}}
\end{equation}

We define
\begin{equation}
     C'_{rst} = 
    \adjustbox{valign = c, scale=0.8}{\begin{tikzpicture}
        \draw (90:1) node[draw,circle] (pc) {$P_{l_c}$};
        \draw (-30:1) node[draw,circle] (pb) {$P_{l_b}$};
        \draw (210:1) node[draw,circle] (pa) {$P_{l_a}$};
        \draw (90:2.3) node (c) {$t$};
        \draw (-30:2.3) node (b) {$s$};
        \draw (210:2.3) node (a) {$r$};
        \draw (pa) edge["\footnotesize$l_a$"'] (a);
        \draw (pb) edge["\footnotesize$l_b$"] (b);
        \draw (pc) edge["\footnotesize$l_c$"] (c);
        \draw (pa) edge[ultra thick] (pb);
        \draw (pb) edge[ultra thick] (pc);
        \draw (pc) edge[ultra thick] (pa);
        \node at (2.25,0) {\textit{or}};
    \end{tikzpicture}}
    \qquad
    \adjustbox{valign = c, scale=0.8}{\begin{tikzpicture}
        \node[circle, fill, inner sep=1.5pt] (fill) at (0,0) {}; 
        \draw (90:1) node[draw,circle] (pc) {$P_{l_c}$};
        \draw (-30:1) node[draw,circle] (pb) {$P_{l_b}$};
        \draw (210:1) node[draw,circle] (pa) {$P_{l_a}$};
        \draw (90:2.3) node (c) {$t$};
        \draw (-30:2.3) node (b) {$s$};
        \draw (210:2.3) node (a) {$r$};        
        \draw (pa) edge["\footnotesize$l_a$"'] (a);
        \draw (pb) edge["\footnotesize$l_b$"] (b);
        \draw (pc) edge["\footnotesize$l_c$"] (c);     
        \draw (pa) edge[ultra thick] (pb);
        \draw (pb) edge[ultra thick] (pc);
        \draw (pc) edge[ultra thick] (pa);
        \draw (fill) edge (pa);
        \draw (fill) edge (pb);
        \draw (fill) edge (pc);
    \end{tikzpicture}}
\end{equation}

To prove that our construct is equivalent to the TP operation, we only need to prove that $C'_{rst}$ is non-zero (obviously $C'_{rst}=1$ when $r,s,t$ are of the highest weight) and is proportional to $C_{rst}$ by a factor independent of $\bm a,\bm b, \bm c$, which is clear since $C'_{rst}$ and $C_{rst}$ are both symmetric rank-3 tensor with indices of representation $l_a,l_b,l_c$, and the symmetric rank-3 tensor with indices of representation $l_a,l_b,l_c$ is unique up to rescaling.
\end{document}